\documentclass[12pt]{article}

\usepackage{amsmath}
\usepackage{amsthm}
\usepackage{url}
\usepackage{amsfonts}
\usepackage{graphicx}
\usepackage{epstopdf}
\usepackage{algorithmic}
\usepackage{hyperref}
\usepackage{delimset}
\usepackage[capitalize,nameinlink]{cleveref}
\usepackage[caption=false]{subfig}
\ifpdf
  \DeclareGraphicsExtensions{.eps,.pdf,.png,.jpg}
\else
  \DeclareGraphicsExtensions{.eps}
\fi

\hypersetup{
  colorlinks=false,
  linkcolor=blue,
  citecolor=blue,
  urlcolor=blue,
  pdfborder={0 0 0}
}

\usepackage[
  left=1in,
  right=1in,
  top=1in,
  bottom=1in
]{geometry}

\usepackage{enumitem}
\setlist[enumerate]{leftmargin=.5in}
\setlist[itemize]{leftmargin=.5in}

\usepackage{amsopn}

\newcommand{\bc}{\boldsymbol{c}}

\newcommand{\be}{\boldsymbol{e}}

\newcommand{\bp}{\boldsymbol{p}}

\newcommand{\bv}{\boldsymbol{v}}

\newcommand{\bA}{\boldsymbol{A}}
\newcommand{\bB}{\boldsymbol{B}}
\newcommand{\bC}{\boldsymbol{C}}
\newcommand{\bD}{\boldsymbol{D}}
\newcommand{\bE}{\boldsymbol{E}}
\newcommand{\bF}{\boldsymbol{F}}

\newcommand{\bH}{\boldsymbol{H}}
\newcommand{\bI}{\boldsymbol{I}}

\newcommand{\bL}{\boldsymbol{L}}

\newcommand{\bP}{\boldsymbol{P}}
\newcommand{\bQ}{\boldsymbol{Q}}

\newcommand{\bU}{\boldsymbol{U}}
\newcommand{\bV}{\boldsymbol{V}}

\newcommand{\bX}{\boldsymbol{X}}

\newcommand{\bZ}{\boldsymbol{Z}}

\newcommand{\bLambda}{\boldsymbol{\Lambda}}

\newcommand{\bPhi}{\boldsymbol{\Phi}}

\newcommand{\bzero}{\boldsymbol{0}}

\DeclareMathOperator{\vect}{vec}

\DeclareMathOperator{\vol}{Vol}
\DeclareMathOperator{\diam}{diam}

\DeclareMathOperator*{\argmin}{arg\,min}

\newtheorem{theorem}{Theorem}
\newtheorem{corollary}{Corollary}
\newtheorem{lemma}{Lemma}
\newtheorem{definition}{Definition}
\newtheorem{proposition}{Proposition}
\newtheorem{assumption}{Assumption}

\crefformat{equation}{\textup{#2(#1)#3}}
\crefrangeformat{equation}{\textup{#3(#1)#4--#5(#2)#6}}
\crefmultiformat{equation}{\textup{#2(#1)#3}}{ and \textup{#2(#1)#3}}
{, \textup{#2(#1)#3}}{, and \textup{#2(#1)#3}}
\crefrangemultiformat{equation}{\textup{#3(#1)#4--#5(#2)#6}}%
{ and \textup{#3(#1)#4--#5(#2)#6}}{, \textup{#3(#1)#4--#5(#2)#6}}{, and \textup{#3(#1)#4--#5(#2)#6}}
\Crefformat{equation}{#2Equation~\textup{(#1)}#3}
\Crefrangeformat{equation}{Equations~\textup{#3(#1)#4--#5(#2)#6}}
\Crefmultiformat{equation}{Equations~\textup{#2(#1)#3}}{ and \textup{#2(#1)#3}}
{, \textup{#2(#1)#3}}{, and \textup{#2(#1)#3}}
\Crefrangemultiformat{equation}{Equations~\textup{#3(#1)#4--#5(#2)#6}}%
{ and \textup{#3(#1)#4--#5(#2)#6}}{, \textup{#3(#1)#4--#5(#2)#6}}{, and \textup{#3(#1)#4--#5(#2)#6}}

\Crefname{figure}{Figure}{Figures}

\begin{document}

\title{Affinity Graph Connectivity in Convex Clustering
}

\author{Sam Rosen\thanks{Department of Statistical Science, Duke University }
\and Jason Xu \thanks{Department of Biostatistics, University of California Los Angeles}
}

\maketitle

\begin{abstract}
    We generalize finite-sample bounds for convex clustering to the setting where affinity weights appearing in the objective correspond to a general connected graph. These bounds and their analysis lead to a better understanding of clustering behavior under various implied connectivity structures behind the data and to new rates of convergence for centroid recovery. The new theoretical framework is based on random walks, which allow application of concentration inequalities related to random graph models, and formalizes the relationship between the clustering performance and the connectivity of the graph structures. Through the form of the bound and empirical results, we argue proper tuning of hyperparameters to convex clustering problems should also include tuning of input affinity weights.
\end{abstract}

\section{Introduction}\label{sec:intro}

Convex clustering, also referred to as clusterpath \cite{hockingClusterpathAlgorithmClustering2011}, relaxes single linkage hierarchical clustering into a convex optimization problem
\begin{equation}
    \hat \bU = \argmin_{\bU \in \mathbb R^{n\times p}} \gamma \sum_{i < j}^n \sqrt{\Phi_{ij}} \|\bU_{i \cdot} - \bU_{j \cdot}\|_2 + \frac{1}{2} \|\bU - \bX\|^2_F, \label{eqn:clusterpath}
\end{equation}
where $\bX \in \mathbb R^{n \times p}$ is a data matrix of $n$ samples and $p$ features and $\|\cdot\|_F$ is the Frobenius norm. The effect of the regularization term are determined by affinities $\Phi_{ij} > 0$, which encourage $\hat \bU_{i\cdot}$ and $\hat \bU_{j\cdot}$ to be close. Hence, the regularization term in \cref{eqn:clusterpath} causes solutions to have a small number of unique rows where $\hat \bU_{i\cdot}$ is interpreted as the cluster centroid of $\bX_{i\cdot}$. If $\hat \bU_{i\cdot} = \hat \bU_{j\cdot}$, then samples $i$ and $j$ are fit to the same cluster. As $\gamma$ increases, the regularization term encourages $\hat \bU$ to coalesce into fewer unique centroids, and eventually leads to a single unique row in $\hat \bU$ equal to the grand mean of $\bX$ \cite{tanStatisticalPropertiesConvex2015}. 

Convex clustering has many appealing properties as well as practical merits. The solution to \cref{eqn:clusterpath} is continuous with respect to $\bX$, $\bPhi$, and $\gamma$, giving stability in the produced centroids and a continuous solution path over $\gamma$ \cite{chiWhyHowConvex2025}. Furthermore, the convexity of the objective function gives a single global optima that is not sensitive to initialization and can be optimized in a variety of manners, including MM algorithms \cite{chiSplittingMethodsConvex2015, landeros2023mm, touwConvexClusteringMM2023} and other scalable variants \cite{panahiClusteringSumNorms2017, sunConvexClusteringModel2021, weylandtDynamicVisualizationFast2020, zhouScalableAlgorithmsConvex2021}. Under any method, a sparse choice of $\bPhi$ can also accelerate optimization  by reducing the computation and storage requirements of the penalty term. Along with these desirable properties, convex clustering is highly extensible by modifying \cref{eqn:clusterpath}, expanding the idea to settings including binary or multi-view data \cite{choiConvexClusteringBinary2019, wangIntegrativeGeneralizedConvex2021}, graph-structured data \cite{donnat2019convex}, supervised learning \cite{wangSupervisedConvexClustering2023}, and robust loss functions \cite{sunResistantConvexClustering2025}. 

In all of these cases, the user needs to specify some choice of $\bPhi$, and the quality of solutions to \cref{eqn:clusterpath} depends heavily on this choice. Typically, the user chooses $\Phi_{ij} \geq 0$ \textit{a priori}, which act as weights on the relative importance of the distance between $\hat \bU_{i\cdot}$ and $\hat \bU_{j\cdot}$. Intuitively, since every $\Phi_{ij} > 0$ encourages samples $i$ and $j$ to share a centroid, $\bPhi$ ideally has many positive entries for pairs of points sharing a true centroid, and small or zero entries for pairs that do not. Following this intuition, many papers \cite{chakrabortyBiconvexClustering2023, chiProvableConvexCoclustering2020, chiSplittingMethodsConvex2015, wangSparseConvexClustering2018} recommend setting affinities to be sparse from a combination of Gaussian kernel and $k$-nearest neighbors calculations
\begin{displaymath}
    \Phi_{ij} = \begin{cases}
        \exp\brk1{-\tau \|\bX_{i\cdot} - \bX_{j\cdot}\|^2_2} & \bX_{i\cdot} \text{ and } \bX_{j\cdot} \text{ are $k$-nearest neighbors}, \\
        0 & \text{otherwise}.
    \end{cases}
\end{displaymath}
These studies have noted that this heuristic performs well in practice, but rigorous support for this observation is lacking in the literature. It is not clear  how densely $\bPhi$ should be chosen in order to recover clusters well. For instance, existing work has noted that under uniform weights $\Phi_{ij} = 1$ for $i \neq j$, convex clustering fails in simple cases such as data supported on two disjoint balls \cite{JMLR:v25:21-0495}, or with a single corrupted data point \cite{sunResistantConvexClustering2025}. When allowing $\bPhi$ to be non-uniform, convex clustering is able to maintain practical guarantees. \cite{sunConvexClusteringModel2021} showed that cluster identification can be recovered if $\bPhi$ satisfies a variety of inequalities related to separation of the data, along with $\Phi_{ij} > 0$ for any $i,j$ sharing a centroid. Despite this, most of the statistical guarantees apply only to the case of uniform affinities, first explored in \cite{tanStatisticalPropertiesConvex2015}, effecting a disconnect between theory and practice. A recent exception is a theoretical study by \cite{dunlapLocalVersionsSumofNorms2022} that establishes under quite technical conditions on the support of $\bX$  that if $\Phi_{ij} = \tau^{p+1} \exp(-\tau \|\bX_{i\cdot} - \bX_{j\cdot}\|_2)$ then the mean squared error of solutions to \cref{eqn:clusterpath} converges to zero in probability. In addition, if $\Phi_{ij}$ is set to zero when $\|\bX_{i\cdot} - \bX_{j\cdot}\|_2 > r$, provided $r$ satisfies regularity conditions, convex clustering remains consistent. In particular, this result implies clustering consistency for two construction methods for $\bPhi$ when the support of $\bX$ is a disjoint union of ``effectively star-shaped" sets.

In this article, we propose a simple new perspective that leads to bounds on the mean squared error of convex clustering solutions under general choices of $\bPhi$. Under weaker assumptions than prior  work, 
we extend the finite-sample bounds in \cite{tanStatisticalPropertiesConvex2015} and identify the key properties of $\bPhi$ which result in high-quality solutions. The ideas employed in \cite{tanStatisticalPropertiesConvex2015} have proven useful to derive bounds for variants of convex clustering  under mild regularity conditions \cite{chakrabortyBiconvexClustering2023,chiProvableConvexCoclustering2020, linLowRank, wangSparseConvexClustering2018}. However, the techniques in \cite{tanStatisticalPropertiesConvex2015} lack flexibility in the affinities $\bPhi$, allowing only $\Phi_{ij} = 1$ for $i \neq j$. In turn, the bound only implies consistency of convex clustering with specific conditions on the data-generating mechanism: cluster centers differing by a fixed set of features, yet $p$ growing with $n$ and $\sqrt{\frac{n \log( n^2 p)}{p}} = o(1)$. For general $\bPhi$, some finite-sample bounds in the specific case of convex clustering on low-rank data matrices exist \cite{linLowRank}, and we standardize these results for objective functions of the form \cref{eqn:clusterpath}. Finally, our results contribute to the broader statistical theory behind convex clustering in establishing new prediction consistency bounds. The proof technique is generalizable and leverages a novel connection to commute times of a random walk on graphs formed by the affinities. These findings are supported by our empirical study, which demonstrates that the properties of this random walk are related to solution quality in practice.

\section{Convex Clustering Setup and Basic Bounds} \label{section:theory}

To facilitate our analysis, here we define the data-generating mechanism and framework for affinity weights. We begin by imposing a minimal set of assumptions for a notion of consistency to be well-defined: there exists a fixed number of unique centers from which the observations (rows) are sampled under any sub-Gaussian distribution. This can be interpreted as a generic mixture model without further parametric assumptions on its components.

\begin{assumption}\label{assumption:data-generation}
    The data $\bX \in \mathbb R^{n\times p}$ is generated from $n$ independent observations of a center added to sub-Gaussian noise. Specifically, $\bX_{i\cdot} = \bU_{i\cdot} + \bE_{i\cdot}$ where $\bU_{i\cdot}$ is selected from $\mathcal C = (\bc_1,\ldots,\bc_m)$ with membership probability $\bp$, and $\bE_{i\cdot}$ is mean-zero with sub-Gaussian norm $\sigma_p$, i.e.
    \begin{displaymath}
        \sigma_p = \|\bE_{i\cdot}\|_{\psi_2} = \sup_{\|\bv\|_2 = 1} \inf \{K > 0 \colon  \mathbb E[\exp((\bv^\top \bE_{i\cdot})^2 / K^2)] \leq 2 \}. \quad (\text{Definition 3.4.1 of \cite{vershyninHighDimensionalProbabilityIntroduction2026}})
    \end{displaymath}
\end{assumption}
This assumption will allow us to discuss clustering consistency in the sense of the mean squared error, $\frac{1}{np} \|\hat \bU - \bU\|_F^2$, converging in probability to zero.

Instead of viewing each affinity  $\Phi_{ij}$ individually, we consider them holistically as a graph $G$. For additional clarity, we define the affinity graph and some derived quantities used throughout this work.
\begin{definition}[Affinity Graph]\label{def:affinity-graph}
    For a data matrix $\bX \in \mathbb R^{n\times p}$, a valid affinity graph, $G = (V, \mathcal E)$, is undirected, connected, and has no negative edge weights or self-loops. Each valid affinity graph corresponds to a symmetric adjacency matrix, $\bPhi \in \mathbb R^{n\times n}$, with nonnegative entries and a zero diagonal. Each node $i \in V$ corresponds to sample $\bX_{i\cdot}$ and $(i, j) \in \mathcal E \iff \Phi_{ij} > 0$. 
\end{definition}
We further define the following quantities for affinity graphs:
    \begin{itemize}
        \item the degree matrix $\bD \in \mathbb R^{n\times n}$, a diagonal matrix where $D_{ii}$ is the degree of node $i$;
        \item an oriented edge-incidence matrix $\bF \in \mathbb R^{n \times |\mathcal E|}$ where each column has two nonzero values corresponding to the two nodes connected by that edge, e.g. $\bF_{j,\mathcal E(j, k)} = -\bF_{k, \mathcal E(j,k)} = \sqrt{\Phi_{jk}}$, where $\mathcal E(j, k)$ is the index of the edge between nodes $j$ and $k$ (without loss of generality, assume $j < k$);
        \item and the Laplacian, $\bL = \bF \bF^\top = \bD - \bPhi$.
    \end{itemize}
We note that imposing connectivity is a mild condition in practice, as multiple disconnected components suggests an immediate partition of the samples. The user can then simply consider further clustering on each component individually, where our contributions apply. Connectivity is immediately satisfied under a variety of graph constructions. For example, a complete graph is trivially connected almost surely, and under commonly used embeddings such as $\epsilon$-radius or $k$-nearest neighbor graphs, conditions which guarantee that $G$ will be connected for large enough datasets are well-studied \cite{maierClusterIdentificationNearestNeighbor2007}.

Using the definition of $\bF$ from \cref{def:affinity-graph}, we can write 
\begin{displaymath}
    \operatorname{col}^\top_{\mathcal E(i, j)}(\bF) \bU = \sqrt{\Phi_{ij}} (\bU_{i\cdot} - \bU_{j\cdot}),
\end{displaymath} allowing us to 
rewrite \cref{eqn:clusterpath} as
\begin{displaymath}
     \min_{\bU \in \mathbb R^{n\times p}} \gamma \sum_{(i, j) \in \mathcal E} \|\bF_{\mathcal E(i, j)\cdot }^\top \bU\|_2 + \frac{1}{2} \|\bU - \bX\|^2_F.
\end{displaymath}
This reformulation is key to analyzing statistical properties of \cref{eqn:clusterpath}. As an example of convex clustering and an affinity graph, \cref{fig:knn-graph} shows 60 points in $\mathbb R^2$ generated from a Gaussian mixture model with three components. Each edge drawn represents some $\Phi_{ij} = 1$ derived from the 4-nearest neighbors of the points. \cref{fig:knn-graph-solution} shows the solution path for a range of $\gamma$, with points colored by the estimated cluster labels for four different clusters when $\gamma \approx 7$. Using a $k$-nearest neighbor graph leads to a convex clustering of the data that recovers many of the underlying labels (an adjusted Rand index of 0.58). Other choices of $\bPhi$ can further improve clustering performance. 

\begin{figure}[htbp]
    \centering
    \subfloat[Input $\bPhi$ and $\bX$]{\label{fig:knn-graph}\includegraphics[width=0.48\textwidth]{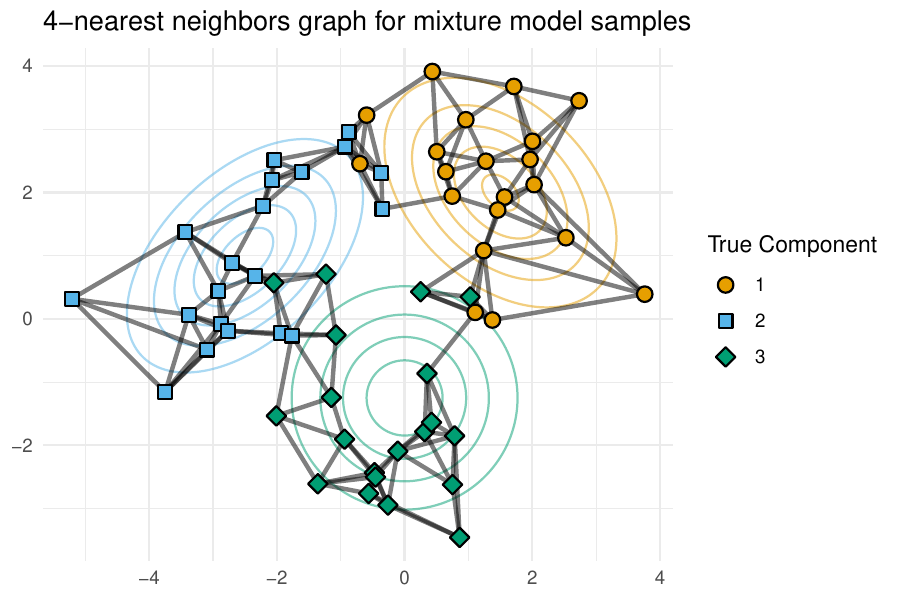}}
    \subfloat[Output solution path over $\gamma$]{\label{fig:knn-graph-solution}\includegraphics[width=0.48\textwidth]{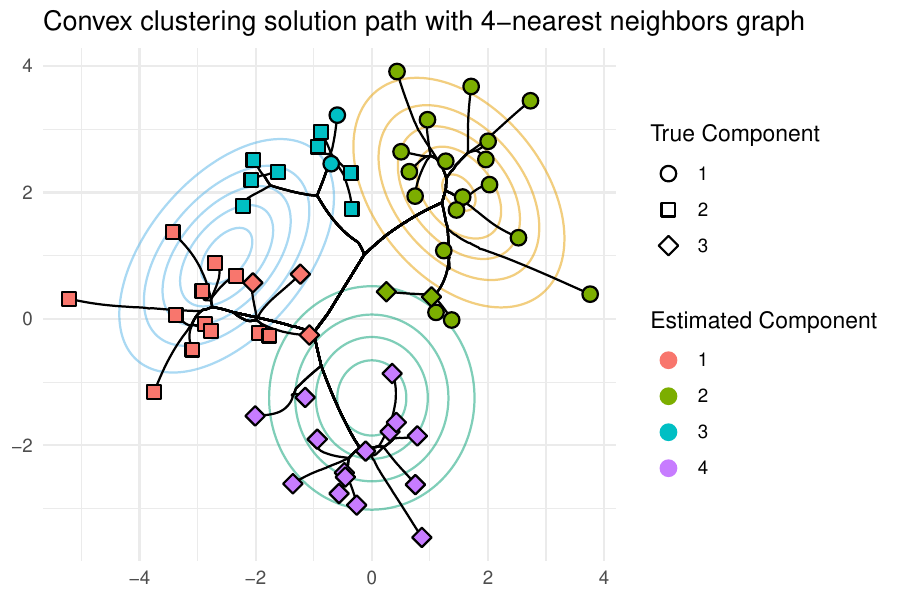}}
    \caption{(a) Points generated from one of three bivariate normals, colored by true cluster, with mixture component density contours shown in the background. Points are connected by being $4$-nearest neighbors. (b) Convex clustering solution path with points colored by estimated component membership at $\gamma \approx 7$, and shaped by true component.}
    \label{fig:knn-graph-both}
\end{figure}

To establish a theoretical framework relating $\bPhi$ to the clustering solution quality, we start with a general version of Lemma 7 from \cite{tanStatisticalPropertiesConvex2015}, describing the mean squared error of $\hat \bU$ with a term strictly dependent on noise and the other on the construction of $\bPhi$. This result applies as long as the underlying graph is connected almost surely, regardless if $\bPhi$ is derived from the data, or fixed beforehand. Similar to \cite{linLowRank}, a version of this theorem could be derived without connectivity almost surely, but we deem this case unhelpful for understanding the dynamics of \cref{eqn:clusterpath} and $\bPhi$, since it reduces to disjoint convex clustering problems.

\begin{theorem} \label{thm:data-dependent-standard} Let $\bX = \bU + \bE \in \mathbb R^{n \times p}$ be generated according to \cref{assumption:data-generation}. Let $\bPhi \in \mathbb R^{n \times n}$ be a symmetric adjacency matrix, corresponding to a valid affinity graph as defined in \cref{def:affinity-graph}, almost surely. If $\hat \bU$ is the global minimum from \cref{eqn:clusterpath} with $\gamma \geq \sqrt{p}\|\bF^\dagger \bE\|_{\max}$ then
    \begin{equation}
        \frac{1}{2np} \|\bU - \hat \bU\|_F^2 \leq c_1\sigma_p^2\brk3{\frac{1}{n} + {\frac{\log(np)}{n p}}} + \frac{2 \gamma}{np} \sum_{(i, j) \in \mathcal E} \sqrt{\Phi_{ij}} \|\bU_{i\cdot} - \bU_{j\cdot} \|_2, \label{eqn:thm-bound}
    \end{equation}
    with probability at least $1 - \frac{1}{np}$ where $c_1 > 0$ is some constant and $\bF^\dagger$ is the Moore--Penrose pseudoinverse of $\bF$.
\end{theorem}

To gain intuition and contextualize the result, we compare this bound to the one established in \cite{tanStatisticalPropertiesConvex2015} for complete, unweighted graphs. Via \cref{thm:data-dependent-standard}, if $\Phi_{ij} = 1$ for all $i \neq j$ and $\gamma = \frac{4\sigma_p \sqrt{p \log(np)}}{n}$ then 
\begin{equation}
    \frac{1}{2np} \|\bU - \hat \bU\|_F^2 \leq c_1\sigma_p^2\brk3{\frac{1}{n} + {\frac{\log(np)}{n p}}} + 8 \sigma_p \sqrt{\frac{\log(np)}{p}}\diam(\mathcal C), \label{eqn:fully-connected-bound}
\end{equation}
with probability at least $1 - \frac{3}{np}$, where $\diam(\mathcal C)$ is the diameter of $\mathcal C$. This bound is tighter than the previous result of  \cite{tanStatisticalPropertiesConvex2015}, implying if $\sigma_p$ remains constant as $p$ grows and $\diam(\mathcal C) = o\brk1{\sqrt{\frac{p}{\log(np)}}}$ then the mean squared error converges to zero in probability. For a fixed $p$, the na\"ive choice of uniform weights is not sufficient for convergence as $n \to \infty$ under the general settings described in \cref{assumption:data-generation}. In \cite{chiWhyHowConvex2025}, the authors note several papers on the recovery abilities of convex clustering supporting this fact, and require the data from each cluster is sufficiently far apart for uniform weights to give a strong answer. 

The upper bound on the mean squared error in \cref{thm:data-dependent-standard} comprises two terms. The first term in  \cref{eqn:thm-bound} is dependent on the error variance and goes to zero at a sublinear rate with respect to $n$. The second term is more nuanced and depends on the product between a term involving the true $\bU$ and a term involving a more esoteric quantity, $\bF^\dagger$. We refer to the factor appearing in the second term, 
\begin{equation}
    \frac{1}{n}\|\bF^\dagger \bE\|_{\max} \sum_{(i, j) \in \mathcal E} \sqrt{\Phi_{ij}} \|\bU_{i\cdot} - \bU_{j\cdot}\|_2 , \label{eqn:oracle-term}
\end{equation}
as the \textit{oracle term} since computing this quantity requires the unknown true centroids and errors. 

This oracle quantity---and in turn the upper bound---increases with the number of $\Phi_{ij} > 0$ when observations $i$ and $j$ do not share the same true centroid. This lends new theoretical support to the widespread choice to initialize convex clustering with sparse yet sensible affinities in practice. From this lens, a sparser graph has fewer positive entries in $\bPhi$,  and a data-driven choice such as $k$-nearest neighbors may help to avoid erroneous (between-cluster) edges while maintaining sparsity. Also, note that the oracle term remains the same when scaling $\bPhi$ by a positive constant. \cref{thm:data-dependent-standard} also provides a sufficient condition for consistency: the mean number of between-cluster edges per node should not outpace $\|\bF^\dagger \bE\|_{\max}$ decreasing. 
Immediately, this yields the following bound for choices of $\bPhi$ that do not depend on $\bE$.

\begin{lemma}\label{lemma:crude-bound} 
Let $\bX \in \mathbb R^{n\times p}$ be generated according to \cref{assumption:data-generation}. Let $\bPhi \in \mathbb \{0, 1\}^{n \times n}$ be an adjacency matrix corresponding to a connected unweighted affinity graph as defined in \cref{def:affinity-graph}. Let the set of between-cluster edges be $\mathcal I = \{(i, j) \colon \bU_{i\cdot} \neq \bU_{j\cdot}, \Phi_{ij} = 1 \}$. If $\bPhi$ is independent of $\bE$, then
    \begin{displaymath}
        \frac{\|\bF^\dagger \bE\|_{\max}}{n} \sum_{(i, j) \in \mathcal E} \sqrt{\Phi_{ij}} \|\bU_{i\cdot} - \bU_{j\cdot} \|_2 \leq 3\sigma_p |\mathcal I| \sqrt{\frac{\log(np)}{n}} \diam(\mathcal C),
    \end{displaymath}
    with probability at least $1 - \frac{2}{np}$.
\end{lemma}
\cref{lemma:crude-bound} is shown by first proving \cref{corollary:independent-F-dagger-E-max}, a concentration inequality for $\|\bF^\dagger \bE\|_{\max}$ when $\bF$ and $\bE$ are independent. The result then follows  by showing $\|\bF^\dagger\|_{\max}$ is bounded by one, using known properties on the commute times of unweighted graphs. Although \cref{lemma:crude-bound} appears to bound mean squared error by a rate of $\sqrt{\frac{\log(n)}{n}}$, this occurs only if the size of $\mathcal I$ is bounded by a constant as $n$ grows. In \cref{sec:bounds}, we show that improved bounds can be derived by additionally considering the within-cluster connectivity.

The number of between-cluster edges $|\mathcal I|$  is not the end of the story: other intrinsic properties of $G$ which influence $\bF^\dagger$ also affect the performance of convex clustering. \cref{fig:knn-bottleneck} suggests that flaws in the clustering solution from the illustrative example in \cref{fig:knn-graph-solution} are associated with the norm of $\bF^\dagger_{\mathcal E(i, j)\cdot}$. In this example, we see several points that should belong to true component 3 but are misclassified due to proximity in both Euclidean distance and the $k$-nearest neighbor edges. Moreover, there is an extra solution component between clusters 1 and 2. In \cref{fig:knn-bottleneck}, there are several high values of $\|\bF^\dagger_{\mathcal E(j,k)\cdot}\|_2$ where points assigned to this erroneous extra cluster connect to the rest of the affinity graph. Inspecting the figure, high values of $\|\bF^\dagger_{\mathcal E(j,k)\cdot}\|_2$ are associated with whether the edge between nodes $j$ and $k$ could be characterized as a ``bottleneck". 

Relating these observations back to \cref{eqn:oracle-term}, we can produce a better solution for the convex clustering problem from \cref{fig:knn-graph-both} with $k = 14$ and $\gamma \approx 0.9$, resulting in an adjusted Rand index of 0.59. Using $k=14$ leads to 121 between-cluster edges while $k=4$ gives only 25. Here, $\|\bF^\dagger \bE\|_{\max} \approx 7.16$ for $k=4$, but  it is only $\|\bF^\dagger \bE\|_{\max} \approx 0.84$ when $k=14$, leading to a smaller oracle term.

\begin{figure}[htbp]
    \centering
    \includegraphics[width=11cm]{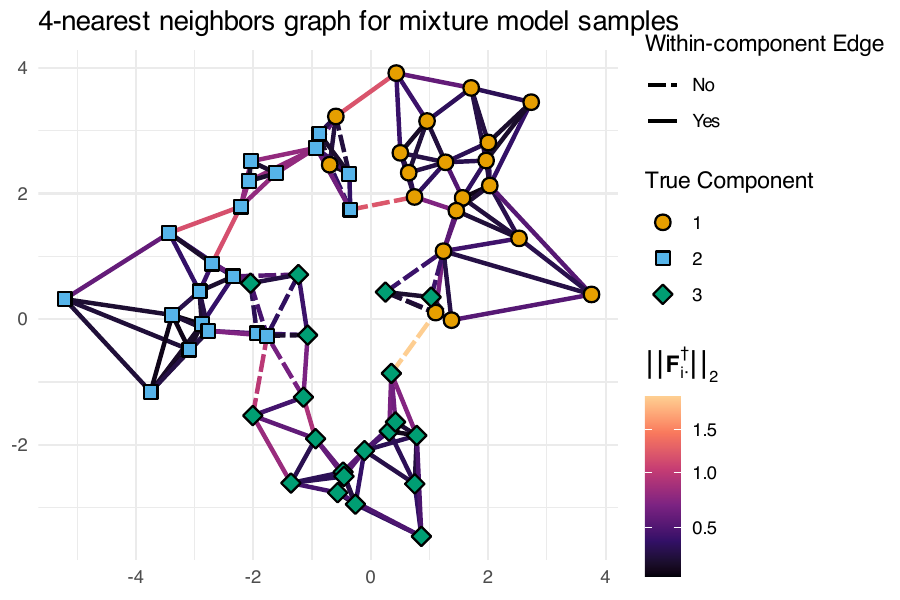}
    \caption{A 4-nearest neighbors affinity graph for 60 points (colored by true component) in $\mathbb R^2$, where edges are colored by $\|\bF^\dagger_{\mathcal E(j, k)\cdot}\|_2$ and are dashed if between points from different components.}
    \label{fig:knn-bottleneck}
\end{figure}

\section{Further Bounds via Commute Times}\label{sec:bounds} The discussion above motivates us to examine properties of $\bF^\dagger$ more closely in order to better understand and bound the oracle term, leading to tighter bounds from \cref{thm:data-dependent-standard}. This section provides a conceptual interpretation of the oracle term, linking convex clustering error to graph bottlenecks concretely. We begin by relating $\bF^\dagger$ to the commute times of a random walk over the corresponding affinity graph.
\begin{lemma}\label{lemma:F-dagger-form}
    Let $\bPhi$ be the symmetric adjacency matrix of an undirected graph $G = (V, \mathcal E)$ with degree matrix $\bD$ and Laplacian $\bL = \bD - \bPhi$. Define the natural random walk over $G$ as the Markov chain with transition matrix $\bP = \bD^{-1} \bPhi$. Then, there exists an oriented edge-incidence matrix $\bF$ such that
    \begin{displaymath}
        \bF^\dagger_{\mathcal E(j, k) \cdot} = \sqrt{\Phi_{jk}} (\bL^\dagger_{\cdot j} - \bL^\dagger_{\cdot k}) = \frac{\sqrt{\Phi_{jk}}}{2\vol(G)} \brk3{\bI_n - \frac{\mathbf 1_n \mathbf 1_n^\top}{n}} (\bC_{\cdot k} - \bC_{\cdot j}),
    \end{displaymath}
    where $\vol(G) = \operatorname{trace}(\bD)$ is the volume of $G$ and $\bC \in \mathbb R_{\geq 0}^{n\times n}$ is the matrix of commute times, i.e. $C_{ij} = H_{ij} + H_{ji}$, with $H_{ij}$ the expected time to travel from node $i$ to node $j$ in the Markov chain.
\end{lemma}
\cref{lemma:F-dagger-form} follows directly from the definition of $\bF$, and the ability to write terms of the  Laplacian pseudoinverse as commute times; this interpretation is described in detail in \cite{vanmieghemPseudoinverseLaplacianBest2017}. The characterization in \cref{lemma:F-dagger-form} explicitly reveals how $\|\bF^\dagger_{\mathcal E(j, k) \cdot}\|_2$ is a bottleneck measure. For an edge in $G$, the quantity describes the average difference in commute times across the entire graph for the nodes sharing the edge. Note the Laplacian $\bL$ is invariant to the choice of orientation for $\bF$, so this norm is invariant as well. The representation of $\bF^\dagger$ in terms of commute times will provide a conduit for finite-sample bounds on the mean squared error of solutions to \cref{eqn:clusterpath} via the oracle term. By bounding the row norms of $\bF^\dagger$, concentration inequalities can be used to bound the oracle term with high probability. This can occur by either explicitly calculating the commute times for an affinity graph model, or appealing to various concentration inequalities for commute times.

\cref{lemma:F-dagger-form} has utility toward strengthening our error bounds for a variety of graph construction methods. \cref{lemma:F-dagger-form} allows us to write $\bF^\dagger$ in \cref{eqn:oracle-term} as a function of the hitting/commute times, which behave in  a more transparent way. \cite{luxburgHittingCommuteTimes2014} show that rescaled hitting times concentrate towards the inverse degree of a node for many random geometric graphs. This is possible by examining spectral properties of the natural random walk, summarized below: 
    \begin{proposition}[Proposition 5 from \cite{luxburgHittingCommuteTimes2014}] \label{prop:luxberg} Let $G$ be a fixed, connected, undirected, possibly weighted graph that is not bipartite, with adjacency matrix $\bPhi$ and degree matrix $\bD$. Then for all $i \neq j$ 
    \begin{equation}
        \bigg|\frac{H_{ij}}{\vol(G)} - \frac{1}{D_{jj}}\bigg| \leq 2\brk3{\frac{1}{1-\lambda_2} + 1}\frac{\|\bPhi\|_{\max}}{\min_k D_{kk}^2}, \label{eqn:luxberg-prop}
    \end{equation}
    where $1 - \lambda_2$ is the second largest eigenvalue of $\bP = \bD^{-1} \bPhi$ and $H_{ij}$ is the expected hitting time from node $i$ to $j$ with transition probability matrix $\bP$.
\end{proposition}
\cref{prop:luxberg} can help describe the concentration of hitting times for $k$-nearest neighbor, $\epsilon$-radius, complete weighted, and several other random graph models relevant to affinities in convex clustering \cite{luxburgHittingCommuteTimes2014}. In particular, for graphs that can be characterized using \cref{prop:luxberg}, we can immediately bound the entries of $\bF^\dagger$ using the form described in \cref{lemma:F-dagger-form}. 
\begin{lemma}\label{lemma:hitting-time-bound}
    Let $G = (V, \mathcal E)$ be a fixed, connected, undirected, possibly weighted graph that is not bipartite, with adjacency matrix $\bPhi$ and degree matrix $\bD$. Let $\bF$ be an oriented edge-incidence matrix for $G$ and $\bP = \bD^{-1}\bPhi$ be the transition matrix for a random walk over $G$. Then
    \begin{align*}
        |\bF^\dagger_{\mathcal E(j, k), i}| & \leq \sqrt{\Phi_{jk}}\brk3{\frac{1}{n} \bigg|{\frac{1}{ D_{kk}} - \frac{1}{ D_{jj}}}\bigg| + 2\max_{\ell \neq m} \bigg|\frac{H_{\ell m}}{\vol(G)} - \frac{1}{D_{mm}}\bigg|}, \quad (i \notin \{j, k\}), \\
        |\bF^\dagger_{\mathcal E(j, k), j}| & \leq \sqrt{\Phi_{jk}} \brk3{\frac{1}{D_{jj}} + \frac{1}{nD_{kk}} + 2 \max_{\ell \neq m} \bigg|\frac{H_{\ell m}}{\vol(G)} - \frac{1}{D_{mm}}\bigg| },
    \end{align*}
    where $H_{ij}$ is the expected hitting time from node $i$ to node $j$ for the random walk.
\end{lemma}

\cref{lemma:hitting-time-bound} is proved in \cref{appendix:lemma:hitting-time-bound} by decomposing $\bF^\dagger$ into hitting times via \cref{lemma:F-dagger-form}, rewriting in terms of the left-hand side of \cref{eqn:luxberg-prop}, and applying the triangle inequality. 
As a consequence of bounding entries of $\bF^\dagger$ in terms of $(1-\lambda_2)^{-1}$ and the degree reciprocals, the trade-offs in the oracle term from \cref{eqn:oracle-term} become clearer. As $G$ becomes more connected, $\lambda_2$ will decrease and the degrees will increase, so the bound on $\bF^\dagger$ grows tighter. However, as the number of edges in $G$ grows, so does the number of potential edges between data points with different centroids.

\subsection{Bounds for Specific Affinity Graphs} 
To illustrate how \cref{lemma:F-dagger-form} is useful toward explicit forms of the bound \cref{eqn:thm-bound} in \cref{thm:data-dependent-standard}, we can consider cases where $\bPhi$ has a known structure that leads to a high-probability worst-case bound for the mean squared error under an ideal choice of affinities. For simplicity, we focus on $\bU$ having only two unique rows and each cluster having the same size.
\begin{corollary}\label{corollary:oracle-graph}
    Assume the data-generating mechanism of Assumption \ref{assumption:data-generation} and suppose $\bU \in \mathbb R^{n \times p}$, $n \geq 6$, has an even number of rows, where $\bU_{1\cdot},\ldots,\bU_{n/2\cdot} = \bc_1$ and $\bU_{n/2+1\cdot},\ldots,\bU_{n\cdot} = \bc_2$. Construct $\bPhi \in \{0, 1\}^{n\times n}$ such that $\Phi_{ij} = 1$ if $\bU_{i\cdot} = \bU_{j\cdot}$. Next, let $\Phi_{i, (i + n / 2)} = 1$ for $1 \leq k \leq n/2$ distinct values of $i$ in $\{1,\ldots,n/2\}$. Then
    \begin{displaymath}
        \frac{\|\bF^\dagger \bE\|_{\max}}{n\sqrt{p}} \sum_{(i, j) \in \mathcal E} \sqrt{\Phi_{ij}} \|\bU_{i\cdot} - \bU_{j\cdot}\|_2 \leq \sigma_p\sqrt{\frac{\log(np)}{p}} \brk3{\sqrt{\frac{2}{n}} + \frac{8k}{n^2}} \|\bc_1 - \bc_2\|_2,
    \end{displaymath}
    with probability at least $1 - \frac{2}{np}$.
\end{corollary}
The assumptions of \cref{corollary:oracle-graph} are that $\bPhi$ contains two complete subgraphs whose edges correspond to nodes that share a centroid, and then $k$ bridges between these subgraphs (see \cref{fig:oracle-graph} for an example). These assumptions identify a mathematically advantageous affinity graph, which has an innate symmetry that allows us to consider a small number of hitting times, so that \cref{lemma:F-dagger-form} is easily applicable. Standard hitting time calculations are applied; under the independence of $\bPhi$ and $\bE$, \cref{corollary:independent-F-dagger-E-max} can be applied immediately. Recovery of cluster labels with graphs matching the structure defined in \cref{corollary:oracle-graph} was first shown via Theorem 5 of \cite{sunConvexClusteringModel2021}, and here we specifically show the solution to \cref{eqn:clusterpath} will converge to the underlying centroids. This advantageous graph allows up to a linear number of between-cluster edges, as opposed to the crude bound in \cref{lemma:crude-bound} which allowed only $o(\sqrt{n})$ between-cluster edges. In \cref{corollary:oracle-graph}, the connectedness within-cluster causes the number of between-cluster edges to be inconsequential and not change the rate of convergence. 

\begin{figure}[htbp]
        \centering
        \includegraphics[width=8cm]{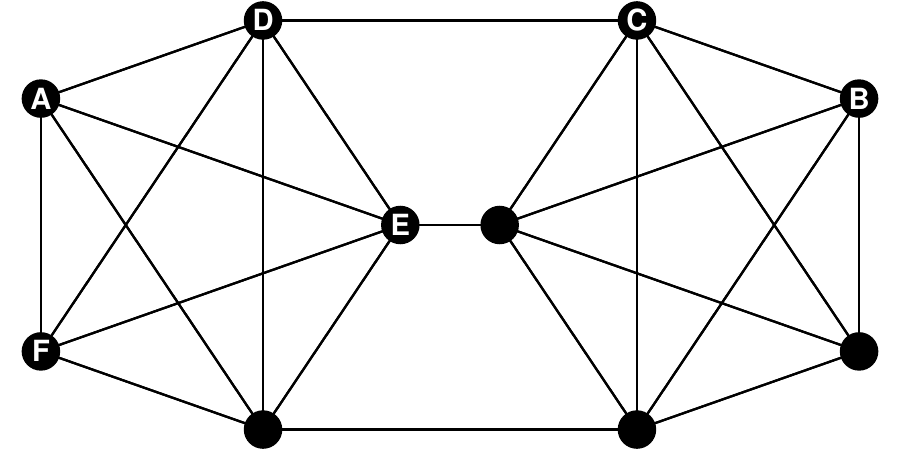}
        \caption{Graph with advantageous information, with complete subgraphs corresponding to within-cluster edges, $n = 10$ and $k = 3$ used in \cref{corollary:oracle-graph}. Due to symmetry, rows of $\bF^\dagger$ only need to be calculated for edges $(A \leftrightarrow F)$, $(A \leftrightarrow D)$, $(D \leftrightarrow E)$, and $(C \leftrightarrow D)$.}
        \label{fig:oracle-graph}
\end{figure}

The graph used for \cref{corollary:oracle-graph} is unrealistic, having complete connectivity within-clusters. A natural question is how dense $\bPhi$ should be within-cluster relative to the number of between-cluster edges to have consistency as $n$ grows. By representing $\bF^\dagger$ in terms of commute times we can appeal to concentration inequalities for commute times that are a consequence of the spectral properties mentioned in \cref{prop:luxberg}. For a wide range of graphs with specific rates on within and between-cluster connection we are able to analyze performance for convex clustering, asymptotically. Again, we assume $\bU$ has two unique clusters of equal size for simplicity, but the results should generalize with extra bookkeeping.

\begin{theorem}\label{thm:edge-crossings-bound}
    Assume the data-generating mechanism of Assumption \ref{assumption:data-generation} and suppose $\bU \in \mathbb R^{n \times p}$, has an even number of rows, where $\bU_{1\cdot},\ldots,\bU_{n/2\cdot} = \bc_1$ and $\bU_{n/2+1\cdot},\ldots,\bU_{n\cdot} = \bc_2$. Let $\bPhi \in \{0, 1\}^{n\times n}$ have zero diagonal and be symmetric, random and independent of $\bE$ such that each entry, $\Phi_{ij}$, $i < j$ is an independent Bernoulli random variable with probabilities
    \begin{displaymath}
        p_w = P(\Phi_{ij} = 1 \mid \bU_{i\cdot} = \bU_{j\cdot}), \qquad p_b = P(\Phi_{ij} = 1 \mid \bU_{i\cdot} \neq \bU_{j\cdot}).
    \end{displaymath}
    If $p_w = \omega\brk1{\frac{\log(n)}{n}}$, $p_b = \omega\brk1{\frac{\log(n)}{n}}$ and $\frac{np_b^2}{p_w + p_b} = \omega(\log(n))$, then 
    \begin{displaymath}
        \frac{\|\bF^\dagger \bE\|_{\max}}{n} \sum_{(i, j) \in \mathcal E} \sqrt{\Phi_{ij}} \|\bU_{i\cdot} - \bU_{j\cdot}\|_2 = O_p\brk3{\sqrt{\log(n)} \brk3{\frac{p_b}{p_w + p_b} + \frac{1}{\sqrt{n}(p_w + p_b)}}}.
    \end{displaymath}
    Hence, if $p_w = \omega\brk1{\sqrt{\frac{\log(n)}{n}}}$ and $p_w = \omega(p_b \sqrt{\log(n)})$, then the oracle term converges to zero in probability.
\end{theorem}
Connectivity will occur almost surely under the conditions of \cref{thm:edge-crossings-bound} (see Theorem 7.3 of \cite{bollobasRandomGraphs2001} and \cite{abbeCommunityDetectionStochastic2018} for further information), allowing application of \cref{thm:data-dependent-standard} for large $n$.
The proof of \cref{thm:edge-crossings-bound} then requires concentration inequalities for $\lambda_2$, $\min_{i} D_{ii}$, and the number of between-cluster edges under the growth model described. See \cref{appendix:thm:edge-crossings-bound} for details. 

\cref{thm:edge-crossings-bound} gives clear growth rates on the rate at which erroneous edges can increase while still guaranteeing consistency of convex clustering solutions. For example, a quadratic number of within-cluster edges, i.e. when $p_w$ is constant, allows for convergence provided $p_b = \omega(1 / n)$ and $p_b = o(p_w / \sqrt{\log(n)})$. 
To apply the spectral bounds from \cref{prop:luxberg}, a superlinear number of between-cluster edges is required, but as shown in \cref{lemma:crude-bound} (which can imply convergence if $p_b = o(n^{-3/2})$) and \cref{corollary:oracle-graph}, this is not necessary for convergence in general. It is intuitive that the number of within-cluster edges should be large enough relative to the between-cluster edges, and \cref{thm:edge-crossings-bound} provides an explicit rate describing when this dominance yields convergence. The result also implies that performance should improve the more $p_w$ dominates $p_b$.

\section{Empirical Results}
\label{sec:experiments}

We now conduct an extensive simulation study to better understand how properties of $\bPhi$ affect solutions to \cref{eqn:clusterpath} in practice. In particular, we examine how solutions to \cref{eqn:clusterpath} recover the true clusters in terms of both centroid and label recovery and how these qualities change as the oracle term from \cref{eqn:oracle-term} increases. Across the simulations, we examine a wide regime of connectedness in $\bPhi$, giving purview into how the oracle term changes as the affinity graph grows denser. Each trial consists of simulating a dataset, constructing many $\bPhi$ for fitting, and then calculating the entire solution path over $\gamma$ for each $\bPhi$\footnote{\texttt{R} code to produce these results is available at \url{https://github.com/SamGRosen/convex_cluster_experiments/}}. 

We consider two constructions for $\bPhi$. First, we take unweighted $k$-nearest neighbor graphs over the entire range of values for $k$, showing the performance of this method for both sparse and dense formulations. Next, we consider the explicit graph forms described in \cref{thm:edge-crossings-bound}, which we will refer to as sparse ``advantageous" graphs. This moniker comes from the graphs being sparse for low values of $p_b$ and $p_w$ and the ability to easily toggle how ``advantageous" a graph is for clustering via the ratio of within to between-cluster connectivity, i.e. $p_w$ vs. $p_b$. In our study we vary both $p_b$ and $p_w$ to be within $[0.02, 0.25]$ to examine how the within and between-cluster connectivity interact to change performance. An example dataset, the 4-nearest neighbor graph derived from it, and the solution path is shown in \cref{fig:trial-graph-both}.

\begin{figure}[htbp]
    \centering
    \subfloat[Input $\bPhi$ and $\bX$]{\label{fig:trial-example}\includegraphics[width=0.48\textwidth]{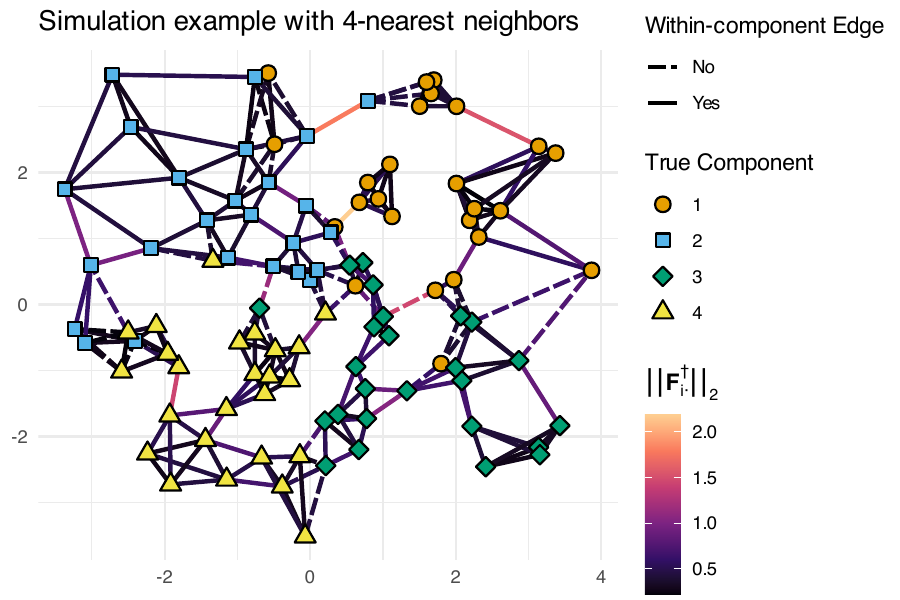}}
    \subfloat[Output solution path over $\gamma$]{\label{fig:trial-solution-path}\includegraphics[width=0.48\textwidth]{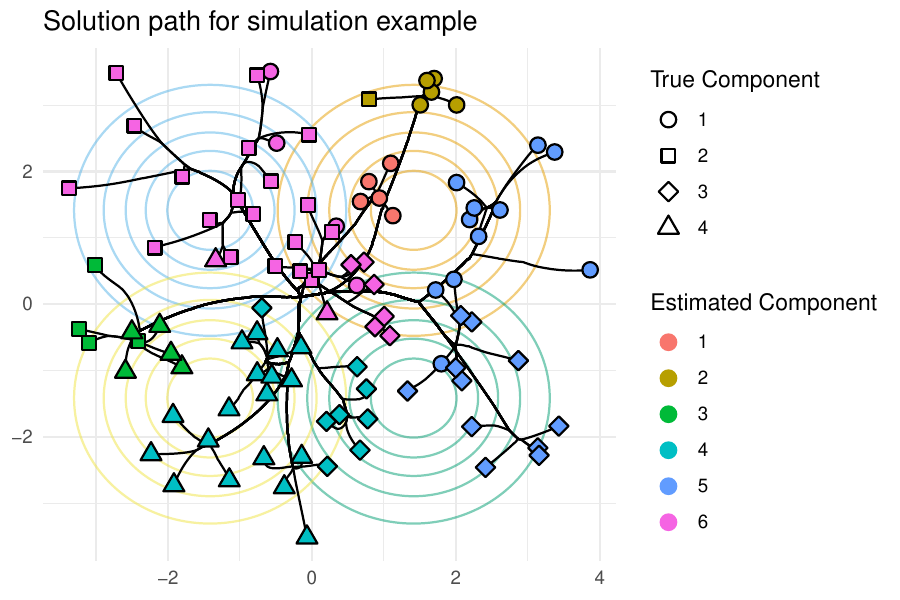}}
    \caption{(a) A 4-nearest neighbors affinity graph for 100 points (colored and shaped by true component) in $\mathbb R^2$ used for one of the trials. For each edge, color signifies $\|\bF^\dagger_{\mathcal E(j, k)\cdot}\|_2$ and dashed if the edge is between points from different clusters. (b) Convex clustering solution path with points colored by estimated component membership at $\gamma \approx 9$, and shaped by true component.}
    \label{fig:trial-graph-both}
\end{figure}

In exact terms, each trial consists of the following simulation steps:
\begin{enumerate}
    \item Generate $\bX = \bU + \bE \in \mathbb R^{100 \times 2}$ where $\bU_{i \cdot} \in \brk[c]3{\begin{bmatrix}
    \sqrt{2} \\ \sqrt{2}
\end{bmatrix}, \begin{bmatrix}
    -\sqrt{2} \\ \sqrt{2}
\end{bmatrix}, \begin{bmatrix}
    \sqrt{2} \\ -\sqrt{2}
\end{bmatrix}, \begin{bmatrix}
    -\sqrt{2} \\ -\sqrt{2}
\end{bmatrix}}$, each center appears 25 times, and entries of $\bE$ are independent standard normal variables.
    \item Construct a sequence of affinity graphs, $\bPhi^{(1)},\ldots,$ with one of the following methods:
    \begin{enumerate}[label=2\alph*., ref=2\alph*]
      \item \textbf{$k$-nearest neighbors:} For all $k \in \{2, \ldots, 99\}$ construct the unweighted $k$-nearest neighbor graph using $\bX$.
      \item \textbf{Sparse advantageous:} For all $(p_w, p_b) \in \{0.02, 0.03,\ldots, 0.25\}^2$, generate 10 samples of $\bPhi$ according to the mechanism of \cref{thm:edge-crossings-bound}.
    \end{enumerate}
    \item Calculate the oracle term from \cref{eqn:oracle-term} for all $\bPhi^{(\cdot)}$.
    \item Let $\hat \bU(\gamma)$ be the solution to \cref{eqn:clusterpath} with hyperparameter $\gamma$. For each $\bPhi^{(\cdot)}$, calculate $\hat \bU(\gamma)$, over a grid of $\gamma$ such that the largest $\gamma$ is large enough for $\hat \bU(\gamma)$ to contain one unique row. 
    \item For each $\bPhi^{(\cdot)}$, record which $\hat \bU(\gamma)$ maximizes the ARI when compared with the true labels. In addition, record which $\hat \bU(\gamma)$ that minimizes $\|\bU - \hat \bU(\gamma)\|_F^2$.
\end{enumerate}
For the sparse advantageous graphs, we perform 400 trials, resulting in a total of 2,304,000 measurements of $\bPhi$, and record the corresponding best ARI and mean squared error. For the $k$-nearest neighbor graphs, we perform 10,000 trials, resulting in a total of 980,000 measurements. 

For each calculated solution to \cref{eqn:clusterpath}, we calculate the mean squared error with the true clusters, and the adjusted Rand index (ARI) of the estimated clusters from the fit. All solutions are calculated with the \texttt{CCMMR} \texttt{R} package \cite{CCMMR}, an implementation of the methods in \cite{touwConvexClusteringMM2023}. ARI is a measure of clustering quality which controls for incidental matching. ARI is between zero, representing a matching similar to randomly guessing, and one, representing a perfect recovery of clusters. We summarize the trials for each affinity graph construction with a heatmap. The heatmaps of \cref{fig:knn-experiments-both} and \cref{fig:sparse-oracle-experiments-both} contain cells that represent the number of affinity graphs that have an oracle term in that bin, along with the best mean squared error or ARI from the solution path. For visualization purposes, both \cref{fig:knn-experiments-oracle-mse} and \cref{fig:sparse-oracle-experiments-mse} discard the trials with an oracle term in the top 0.5\% quantile.

For the $k$-nearest neighbor graphs, \cref{fig:knn-experiments-oracle-mse} shows a clear correlation between mean squared error---the left-hand side of \cref{eqn:thm-bound}---and the oracle term. In general, the best fits have a small oracle term, representing a small number of between-cluster connections, while maintaining high connectivity overall. The performance with respect to $k$ is shown in \cref{fig:knn-experiments-metrics}. If $k$ is too small, the oracle term becomes large due to poor connectivity of the graph, in turn leading to both the mean squared error and ARI showing poor performance. We observed that peak performance both in terms of mean squared error as well as ARI occurs at roughly $k\approx 20$, coinciding with a dip in the oracle term. As $k$ becomes too large, too many between-cluster edges increase the oracle term, and both mean squared error and ARI drop again rapidly. Because the number of between-cluster connections is increasing with $k$, we can conclude that the range of values $k$ for which the oracle term decreases, coincides with the $\|\bF^\dagger \bE\|_{\max}$ term decreasing faster. These empirical results illustrate not only how the theory we derive is consistent with what we observe in practice, but also suggests that by tuning $k$ appropriately, one can maximize the performance of convex clustering.

\begin{figure}[htbp]
    \centering
    \subfloat[]{\label{fig:knn-experiments-oracle-mse}\includegraphics[width=0.52\textwidth]{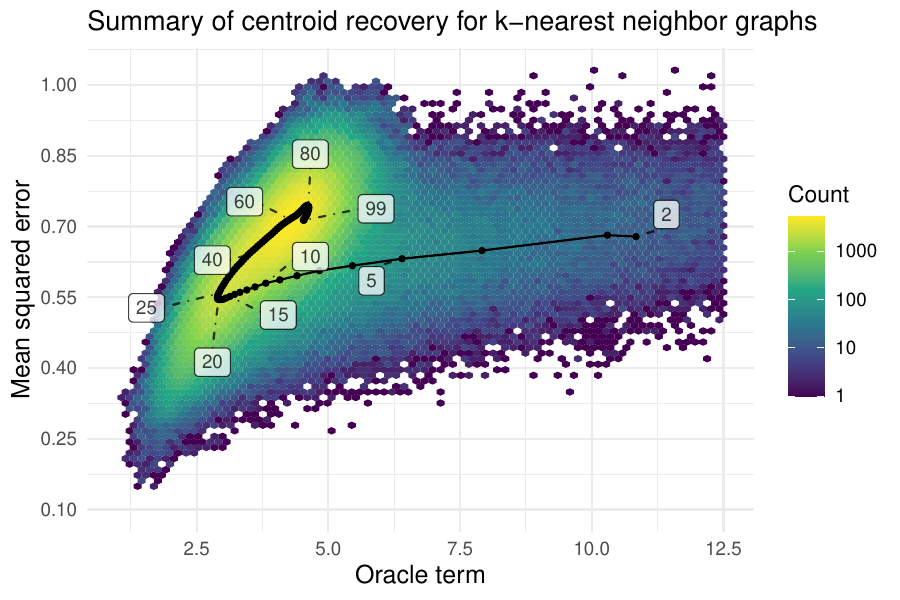}}
    \subfloat[]{\label{fig:knn-experiments-metrics}\includegraphics[width=0.43\textwidth]{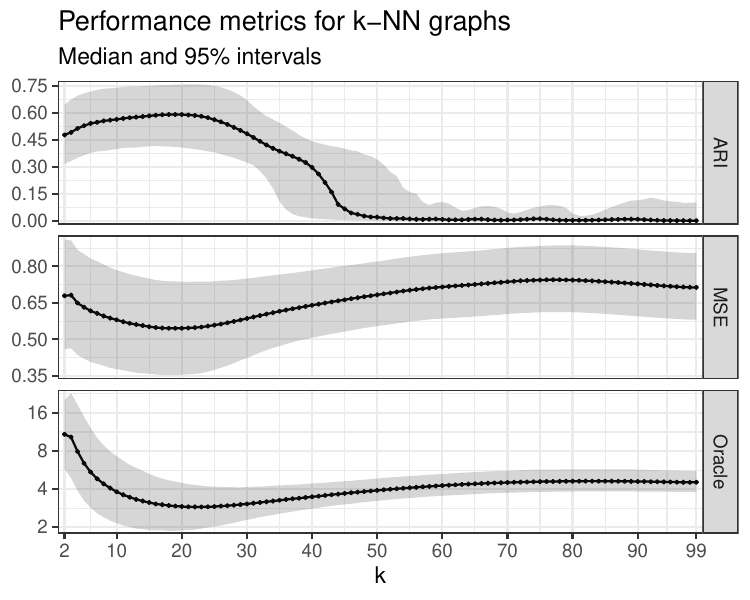}}
    \caption{ (a) For $k$-nearest neighbor trials, heatmap of the oracle term from step 3 and the mean squared error. (b) Median metrics for constructions with $k$-nearest neighbors, in terms of ARI, mean squared error, and the oracle term. Confidence bands include 95\% of observations. The median MSE and ARI values from (b) are superimposed and labeled on (a).}
    \label{fig:knn-experiments-both}
\end{figure}

For the sparse advantageous graphs, \cref{fig:sparse-oracle-experiments-mse} shows a similar story to \cref{fig:knn-experiments-oracle-mse}. As the oracle term grows, the general performance via mean squared error decreases. This matches the intuition behind \cref{thm:data-dependent-standard}. Although this theory was only a worst-case bound and not necessarily tight, there is still empirical evidence of a strong relationship between the oracle term and the quality of solutions to \cref{eqn:clusterpath}. In addition, the oracle term maintains a mostly monotonic correlation with ARI, as shown in \cref{fig:sparse-oracle-experiments-ari}. A small oracle term corresponds to a perfect recovery of clusters, while a large oracle typically has an ARI near zero. As evidenced by \cref{thm:data-dependent-standard}, the oracle term is relevant even for graphs dependent on the data and we show it maintains relevance empirically for $k$-nearest neighbor graphs. \cref{thm:edge-crossings-bound} was also proven through bounds on \cref{eqn:oracle-term}, and we show that the oracle term properly measures performance for the wide balance of within and between-cluster connectivity afforded by the theorem.

\begin{figure}[htbp]
    \centering
    \subfloat[]{\label{fig:sparse-oracle-experiments-mse}\includegraphics[width=0.48\textwidth]{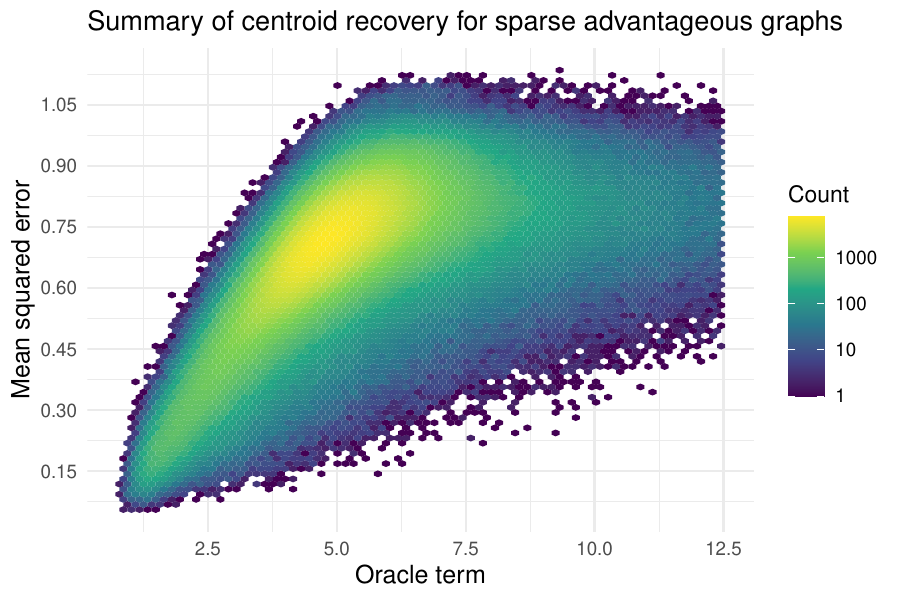}}
    \subfloat[]{\label{fig:sparse-oracle-experiments-ari}\includegraphics[width=0.48\textwidth]{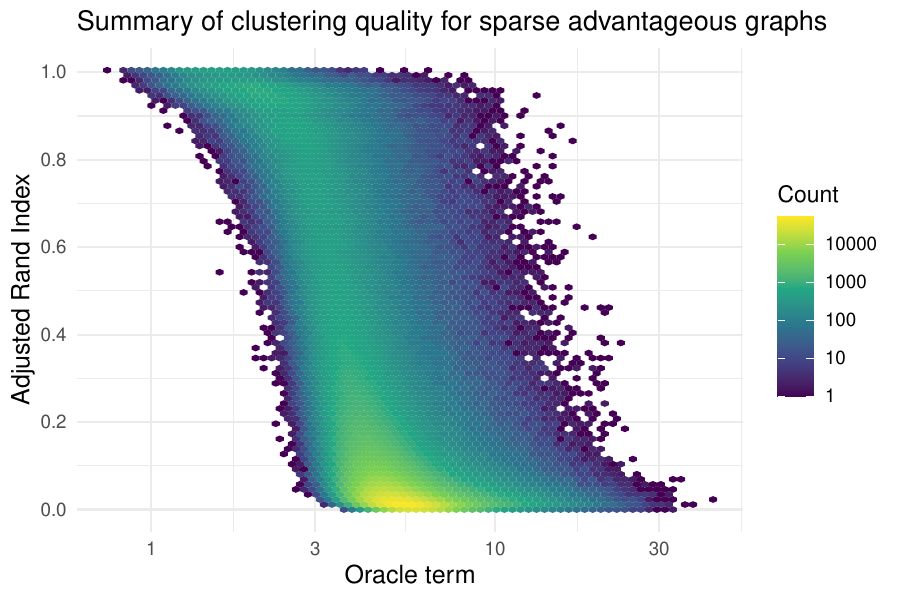}}

    \caption{(a) For sparse advantageous trials, heatmap of the oracle term from step 3 and the mean squared error. (b) Similar to (a), but with the adjusted Rand index measured.}
    \label{fig:sparse-oracle-experiments-both}
\end{figure}

\section{Discussion and Future Directions}
\label{sec:conclusions}
This article has contributed a framework for understanding the theoretical performance of convex clustering as it relates to the affinity weights appearing in the objective \eqref{eqn:clusterpath}. In particular, we find that the connectivity of the underlying graph is a key relevant quantity, which is formally related to clustering performance by bridging analytical techniques from the convex clustering literature to results for random graphs and commute times of random walks. There is extensive literature on both these topics \cite{devriendtEffective, ellensEffectiveGraphResistance2011, ghoshMinimizingEffectiveResistance2008, maierClusterIdentificationNearestNeighbor2007, maierInfluenceGraphConstruction2008, maierHowResultGraph2013} and further application may lead to tighter bounds or extensions of these results to related settings.

These contributions lead to a generalization of finite-sample error bounds to settings that closely align with how convex clustering and its variants are used in practice. At a high level, our theoretical results now match the intuition of convex clustering, while previous results under uniform affinity weights effected a disconnect with the empirical trends observed and well-documented by practitioners. The core quantities identified in our analysis $\bF^\dagger \bE$ and $\sum \sqrt{\Phi_{ij}} \|\bU_{i\cdot} - \bU_{j\cdot}\|_2$ quantify the intuition that high within-cluster connectivity and sparse connectivity between-clusters leads to better performance. For example, \cref{thm:edge-crossings-bound} implies that convex clustering can remain consistent even in the presence of many between-cluster edges, provided that within-cluster connectivity is sufficiently high. The existing analyses in the convex clustering literature assume fully-connected, uniform affinities  \cite{chakrabortyBiconvexClustering2023, chiProvableConvexCoclustering2020, wangSparseConvexClustering2018} and require $p$ to grow with $n$ or specific qualities of the data-generating mechanism to establish consistency. The theoretical framework here could be applied to these works, generalizing their theorems while giving insight into the effect of affinities in these different settings.

In practice, the choice of $\bPhi$ is not independent of $\bE$, and it may be possible to strengthen our results to apply in this case by leveraging dependence structure. Our current work exploits concentration inequalities for $\bF^\dagger$ and $\bE$ separately, and it is not clear how to bound the entries of $\bF^\dagger \bE$ when $\bF^\dagger$ is correlated with $\bE$ through $\bPhi$ being a function of the data. By using results from \cite{luxburgHittingCommuteTimes2014} for common constructions of data-dependent $\bPhi$ and applying the two sets of concentration inequalities with a union bound, the oracle term becomes bounded by a constant with high probability, similar to the case for fully uniform affinities. Some additional theoretical works that may be helpful to bridge this gap include convergence of the spectrum of the Laplacian matrix for $k$-nearest neighbor and $\epsilon$-radius graphs \cite{calderImprovedSpectralConvergence2022} and a variety of results related to the number of between-cluster edges, or cuts, for random graphs constructed from data \cite{maierHowResultGraph2013, narayananRelationLowDensity2006, trillosGraphCutsIsoperimetric2022}.

We find that to acquire the best possible clustering from a fixed dataset, one should tune over a range of affinity graphs along with tuning over $\gamma$. Compared to tuning over $\gamma$ with a fixed $\bPhi$, it may be equally as valid to tune over $\bPhi$ with fixed $\gamma$. Existing consistency results for $\bPhi$ depending on the data from \cite{dunlapLocalVersionsSumofNorms2022} suggest this, which is now further supported by the trade-off implied by \cref{eqn:oracle-term}:  additional between-cluster edges could be offset by gains in within-cluster connectivity. In \cref{sec:experiments}, we find empirically that the solution quality is highly dependent on the structure of $\bPhi$ for any given $\bX$. Existing algorithms for calculating the solution path over $\gamma$ for \cref{eqn:clusterpath} with fixed $\bPhi$ benefit from using the previous solution as a warm start to efficiently calculate the solution path. A similar idea could be used to calculate the optimal solution for increasingly dense values of $\bPhi$ with $\gamma$ fixed. Better yet, a state-of-the-art convex clustering package would calculate the solution path over $\gamma$ and $\bPhi$ jointly, giving a wide range of high quality solutions to choose from. Approaching \cref{eqn:clusterpath} as a parametric optimization problem of $\bPhi$ could lead to fast optimization over this grid of hyperparameters \cite{bonnansPerturbationAnalysisOptimization2000}.

\newpage
\appendix

\section{Proofs of main results}

\subsection{Proof of Theorem \ref{thm:data-dependent-standard}}

We show a corollary of the Hanson-Wright inequality and a property of the graph incidence matrix before proceeding with the proof of \cref{thm:data-dependent-standard} which has many identical steps to the proofs in \cite{tanStatisticalPropertiesConvex2015}. 
\begin{lemma}\label{lemma:modified-hanson-wright}
    Suppose $\bE \in \mathbb R^{n\times p}$ where rows $\bE_{i\cdot}$ are independent, mean-zero and satisfy $\|\bE_{i\cdot}\|_{\psi_2} = \sigma_p$. Then
    \begin{equation*}
        P\brk3{\vect(\bE)^\top \brk3{\bI_p \otimes \frac{\mathbf 1_n \mathbf 1_n^\top}{n}} \vect(\bE) \geq C \sigma^2_p(p + \log(np))} \leq \frac{1}{np},
    \end{equation*}
    where $C > 0$ is some universal constant.
\end{lemma}

\begin{proof}
    We first calculate the sub-Gaussian norm of $\bE$. By definition, 
    \begin{displaymath}
        \|\vect(\bE)\|_{\psi_2}^2 = \sup_{\|\bv\|_2 = 1} \|\bv^\top \vect(\bE)\|_{\psi_2}^2 = \sup_{\|\bV\|_F = 1} \bigg\|\sum_{i=1}^n \sum_{j=1}^p E_{ij} V_{ij} \bigg\|_{\psi_2}^2.
    \end{displaymath}
    
    Via Proposition 2.7.1 of \cite{vershyninHighDimensionalProbabilityIntroduction2026}, since $X_i = \sum_{j = 1}^p E_{ij} V_{ij}$ is independent of $X_k$, $i\neq k$,  
    \begin{align*}
        \sup_{\|\bV\|_F = 1} \bigg\|\sum_{i=1}^n \sum_{j=1}^p E_{ij} V_{ij} \bigg\|_{\psi_2}^2 & \leq \sup_{\|\bV\|_F = 1} C \sum_{i=1}^n \bigg\|\sum_{j=1}^p E_{ij} V_{ij} \bigg\|_{\psi_2}^2, \\
        & = \sup_{\|\bV\|_F = 1} C \sum_{i=1}^n \|\bV_{i\cdot}\|_2^2 \bigg\|\bE_{i\cdot}^\top \frac{\bV_{i \cdot}}{\|\bV_{i\cdot}\|_2} \bigg\|_{\psi_2}^2. \\
        & \leq \sup_{\|\bV\|_F = 1} C \sum_{i=1}^n \|\bV_{i\cdot}\|_2^2 \sup_{\|\bv\|_2 = 1} \|\bE_{i\cdot}^\top \bv \|_{\psi_2}^2.
    \end{align*}
    By definition, $\sup_{\|\bv\|_2 = 1} \|\bE_{i\cdot}^\top \bv \|_{\psi_2}^2 = \sigma_p$, and $\|\bV\|_F^2 = \sum_{i=1}^n \|\bV_{i\cdot}\|_2^2 = 1$, so
    \begin{align*}
        \|\vect(\bE)\|_{\psi_2}^2    \leq C \sigma_p^2.    
    \end{align*}
    As a projection matrix of rank $p$, $\bI_p \otimes \frac{\mathbf 1_n \mathbf 1_n^\top}{n}$ has trace $p$ and unit operator norm. Furthermore, it is positive semidefinite. Hence, we can apply Exercise 6.11 of \cite{vershyninHighDimensionalProbabilityIntroduction2026} for
    \begin{displaymath}
        P\brk3{\vect(\bE)^\top \brk3{\bI_p \otimes \frac{\mathbf 1_n \mathbf 1_n^\top}{n}} \vect(\bE) \geq C \sigma^2_p (p + t)} \leq \exp(-t).
    \end{displaymath}
    Choose $t = \log(np)$ to complete the proof.
\end{proof}

\begin{lemma}[Affinity Graph Incidence Matrix] \label{lemma:incidence-null-space}
    Let $\bF \in \mathbb R^{n \times |\mathcal E|}$ be an oriented edge-incidence matrix for a graph $G = (V, \mathcal E)$, with $n$ nodes. If $G$ is connected, then the projection matrix onto the null space of $\bF^\top$ is $\frac{\mathbf 1_n \mathbf 1_n^\top}{n}$.
\end{lemma}

\begin{proof}
    The sums of all rows of $\bF^\top$ are zero, i.e. $\bF^\top \mathbf 1_n = \bzero$. Since $G$ is connected, $\bF^\top$ has rank $n-1$ (Theorem 9.6 of \cite{deoGraphTheoryApplications1974}) and the null space of $\bF^\top$ has dimension $1$ with $\mathbf 1_n / \sqrt{n}$ as the only basis unit vector. The projection matrix onto this space is $\frac{\mathbf 1_n \mathbf 1_n^\top}{n}$.
\end{proof}

The above proof can also be easily generalized to $G$ having multiple connected components. We now proceed with the proof of \cref{thm:data-dependent-standard}.

\begin{proof}
Let $\mathcal E = \{(i, j) \mid \Phi_{ij} > 0, i < j\}$ be the set of all edges for the row affinities. Let $\bF \in \mathbb R^{n \times |\mathcal E|}$ be the oriented edge-incidence matrix of $\bPhi$ such that $[\bF^\top \bU]_{\mathcal E(i,j)\cdot} = \sqrt{\Phi_{ij}}(\bU_{i\cdot} - \bU_{j\cdot})$ if $\Phi_{ij} > 0$ where $\mathcal E(i, j)$ is also an index set to map an edge to the row it corresponds to in $\bF^\top$. Since $G$ is connected, we know we know $\bF$ has rank $n-1$. Let $\bF^\top = \bQ \bLambda \bV_{\beta}^\top$ be the singular value decomposition of $\bF^\top$. Explicitly, $\bQ \in \mathbb R^{|\mathcal E| \times (n-1)}$, $\bLambda \in \mathbb R^{(n-1)\times (n-1)}$ and $\bV_\beta \in \mathbb R^{n \times (n-1)}$. Via properties of the SVD, we can write $\bV = [\bV_\alpha, \bV_\beta] \in \mathbb R^{n \times n}$ where $\bV_\alpha$ and $\bV_\beta$ are orthonormal matrices such that $\bV_\alpha^\top \bV_\beta = \bzero$. Let the following symbols be defined
\begin{displaymath}
    \bB = \bV_\beta^\top \bU, \quad
    \bA = \bV_\alpha^\top \bU, \quad
    \bZ = \bQ \bLambda.
\end{displaymath}
These symbols imply
\begin{displaymath}
    \bU = \bV_\beta \bB + \bV_\alpha \bA, \quad
    \bF^\top \bU = \bQ \bLambda \bV_\beta^\top \bU = \bZ \bB.
\end{displaymath}

Let $\gamma' = \frac{\gamma}{np}$. Then objective \cref{eqn:clusterpath} is equivalent to
\begin{equation}
    G(\bA, \bB) = \frac{1}{2np} \|\bX - \bV_\alpha \bA - \bV_\beta \bB\|_F^2 + \gamma' \sum_{(i, j) \in \mathcal E} \|\bZ_{\mathcal E(i, j) \cdot } \bB\|_2.
    \label{eqn:equiv_obj}
\end{equation}

Let $\hat \bU$ be the global minimizer of \cref{eqn:clusterpath}. Then by definition
\begin{equation}
    \begin{split}
    & \frac{1}{2np} \|\hat \bU - \bX\|_F^2 + \gamma' \sum_{(i, j) \in \mathcal E} \sqrt{\Phi_{ij}} \|\hat \bU_{i\cdot} - \hat \bU_{j\cdot}\|_2 \leq \frac{1}{2np} \|\bU - \bX\|_F^2 + \gamma' \sum_{(i, j) \in \mathcal E} \sqrt{\Phi_{ij}}\| \bU_{i\cdot} - \bU_{j\cdot}\|_2.
    \end{split} \label{eqn:minima_diff}
\end{equation}

Recall $\bX = \bU + \bE$ and substitute into \cref{eqn:minima_diff}
\begin{equation*}
    \begin{split}
    & \frac{1}{2np} \|\bU + \bE - \hat \bU\|_F^2 + \gamma' \sum_{(i, j) \in \mathcal E} \sqrt{\Phi_{ij}} \|\hat \bU_{i\cdot} - \hat \bU_{j\cdot}\|_2 \leq \frac{1}{2np} \|\bE\|_F^2 + \gamma' \sum_{(i, j) \in \mathcal E} \sqrt{\Phi_{ij}} \| \bU_{i\cdot} - \bU_{j\cdot}\|_2.
    \end{split}
\end{equation*}

Now expand all norms,
\begin{align*}
    \begin{split}
    & \frac{1}{2np} \sum_{\ell=1}^p \|\bU_{\cdot \ell}- \hat \bU_{\cdot \ell} \|_2^2 + \|\bE_{\cdot \ell}\|_2^2 + 2 \bE_{\cdot \ell}^\top (\bU_{\cdot \ell} - \hat \bU_{\cdot \ell}) + \gamma' \sum_{(i, j) \in \mathcal E} \sqrt{\Phi_{ij}} \|\hat \bU_{i\cdot} - \hat \bU_{j\cdot}\|_2  \\
    & \leq \frac{1}{2np} \sum_{\ell=1}^p \|\bE_{\cdot \ell}\|_2^2 + \gamma' \sum_{(i, j) \in \mathcal E} \sqrt{\Phi_{ij}} \| \bU_{i\cdot} - \bU_{j\cdot}\|_2,
    \end{split}
\end{align*}
and collect like terms,
\begin{align}
    \begin{split}
    & \frac{1}{2np} \sum_{\ell=1}^p \|\bU_{\cdot \ell}- \hat \bU_{\cdot \ell} \|_2^2 + \gamma' \sum_{(i, j) \in \mathcal E} \sqrt{\Phi_{ij}} \|\hat \bU_{i\cdot} - \hat \bU_{j\cdot}\|_2  \\
    & \leq \frac{1}{np} \sum_{\ell=1}^p \bE_{\cdot \ell}^\top (\hat \bU_{\cdot \ell} - \bU_{\cdot \ell}) + \gamma' \sum_{(i, j) \in \mathcal E} \sqrt{\Phi_{ij}} \| \bU_{i\cdot} - \bU_{j\cdot}\|_2.
    \end{split} \label{eqn:inequality_to_bound}
\end{align}

Let $\hat \bA, \hat \bB$ be the argument minima of $G$ from $\cref{eqn:equiv_obj}$ so $\hat \bU = \bV_\alpha \hat \bA + \bV_\beta \hat \bB$. We can rewrite \cref{eqn:inequality_to_bound} as
\begin{equation}
    \begin{split}
    & \frac{1}{2np} \|\bV_\alpha ( \bA - \hat \bA) + \bV_\beta ( \bB - \hat \bB) \|^2_{F} + \gamma' \sum_{(i, j) \in \mathcal E} \|\bZ_{\mathcal E(i, j) \cdot } \hat \bB\|_2 \\
    & \leq \frac{1}{np} \sum_{\ell = 1}^p \bE_{\cdot \ell}^\top([\bV_\alpha ( \hat \bA -  \bA)]_{\cdot \ell} + [\bV_\beta (\hat \bB - \bB)]_{\cdot \ell})  + \gamma' \sum_{(i, j) \in \mathcal E} \|\bZ_{\mathcal E(i, j) \cdot } \bB\|_2, \\ 
    & = \frac{1}{np} \operatorname{trace}(\bE^\top \bV_{\alpha} (\hat \bA - \bA)) +  \operatorname{trace}(\bE \bV_{\beta} (\hat \bB - \bB)) + \gamma' \sum_{(i, j) \in \mathcal E} \|\bZ_{\mathcal E(i, j) \cdot } \bB\|_2.
    \end{split} \label{eqn:inequality_rewritten}
\end{equation}
To be a minima, the gradient with respect to $\bA$ of \cref{eqn:equiv_obj} must equal 0 at $\hat \bA$,
\begin{displaymath}
    0 = \nabla_{\bA} \brk3{ \frac{1}{2np} \|\bX - \bV_\alpha \bA - \bV_\beta \bB \|^2_{F}} = \frac{1}{np} \bV_\alpha^\top ( \bX - \bV_\alpha \bA - \bV_\beta \bB).
\end{displaymath}
Since $\bV_\alpha^\top \bV_{\alpha} = 1$ and $\bV_{\beta}^\top \bV_{\alpha} = 0$, $\hat \bA$ satisfies 
\begin{align}
    \hat \bA & = \bV_\alpha^\top \bX = \bV_{\alpha}^\top(\bU + \bE) = \bV_{\alpha}^\top(\bV_{\alpha} \bA + \bV_{\beta} \bB + \bE) = \bA + \bV_{\alpha} \bE \implies \nonumber \\
    \bV_{\alpha}(\hat \bA - \bA) & = \bV_{\alpha} \bV_{\alpha}^\top \bE. \label{eqn:alpha-diff-sub}
\end{align}

We simplify the right-hand side of \cref{eqn:inequality_rewritten} with \cref{eqn:alpha-diff-sub} as
\begin{equation}
    \begin{split}
    & \frac{1}{np} \operatorname{trace}(\bE^\top \bV_{\alpha} (\hat \bA - \bA)) +  \operatorname{trace}(\bE \bV_{\beta} (\hat \bB - \bB)) + \gamma' \sum_{(i, j) \in \mathcal E} \|\bZ_{\mathcal E(i, j) \cdot } \bB\|_2 \\
    & = \frac{1}{np} \underbrace{\operatorname{trace}(\bE^\top \bV_{\alpha} \bV_{\alpha}^\top \bE)}_{(I)} + \frac{1}{np}\underbrace{\operatorname{trace}(\bE \bV_{\beta} (\hat \bB - \bB))}_{(II)} + \gamma' \sum_{(i, j) \in \mathcal E} \|\bZ_{\mathcal E(i, j) \cdot} \bB\|_2. \label{eqn:inequality_simplified}
    \end{split}
\end{equation}

Via properties of the SVD, $\bV_\alpha \bV_\alpha^\top$ is a projection matrix onto the null space of $\bF^\top$. For bounding $(I)$, via \cref{lemma:incidence-null-space}, we have that $\bV_\alpha \bV_\alpha^\top$ is not dependent on $G$ provided it is connected and 
\begin{equation*}
    \operatorname{trace}(\bE^\top \bV_{\alpha} \bV_{\alpha}^\top \bE) = \frac{1}{n}\operatorname{trace}(\bE^\top \mathbf 1_n \mathbf 1_n^\top \bE) = \vect(\bE)^\top \brk3{\bI_p \otimes \frac{\mathbf 1_n \mathbf 1_n^\top}{n}} \vect(\bE).
\end{equation*}
Now apply \cref{lemma:modified-hanson-wright} giving
\begin{equation}
    P\brk3{\frac{1}{np} \operatorname{trace}(\bE^\top \bV_{\alpha} \bV_{\alpha}^\top \bE) \geq \frac{C \sigma_p^2}{np}\brk1{\log(np) + p} } \leq \frac{1}{np} \label{eqn:hanson-wright-bound}.
\end{equation}

Since $G$ is connected, $\bF$ has rank $n-1$ (Theorem 9.6 of \cite{deoGraphTheoryApplications1974}). Hence, $\bZ$ has full column rank, and there exists a pseudoinverse such that $\bZ^\dagger \bZ = \bI_{n-1}$. Insert this expression into ($II$) then apply the triangle and Cauchy-Schwarz inequalities
\begin{align*}
    |\operatorname{trace}(\bE^\top \bV_\beta ( \hat \bB - \bB))| & = |\operatorname{trace}(\bE^\top \bV_\beta \bZ^\dagger \bZ (\hat \bB - \bB))|, \\
    & = \abs3{\operatorname{trace}\brk3{\bE^\top \bV_\beta\brk3{\sum_{(j, k) \in\mathcal E} \bZ^\dagger_{\cdot \mathcal E(j, k)} \bZ_{\mathcal E(j, k)\cdot}}(\hat \bB - \bB)}}, \\
    & = \abs3{\sum_{(j, k) \in\mathcal E} \operatorname{trace}\brk1{\bE^\top \bV_\beta \bZ^\dagger_{\cdot \mathcal E(j, k)} \bZ_{\mathcal E(j, k)\cdot}(\hat \bB - \bB)}}, \\
    & \leq \sum_{(j, k) \in\mathcal E} \abs{ \operatorname{trace}\brk{\bE^\top \bV_\beta \bZ^\dagger_{\cdot \mathcal E(j, k)} \bZ_{\mathcal E(j, k)\cdot}(\hat \bB - \bB)}}, \\
    & \leq \sum_{(j, k) \in\mathcal E}  \|\bE^\top \bV_\beta \bZ^\dagger_{\cdot \mathcal E(j, k)}\|_2 \| \bZ_{\mathcal E(j, k)\cdot}(\hat \bB - \bB)\|_2, \\
    & \leq \max_{(j, k) \in \mathcal E} \|\bE^\top \bV_\beta \bZ^\dagger_{\cdot \mathcal E(j, k)}\|_2 \sum_{(j, k) \in\mathcal E} \|   \bZ_{\mathcal E(j, k)\cdot}(\hat \bB - \bB)\|_2.
\end{align*}
By definition, $(\bF^\top)^\dagger = \bV_{\beta}\bZ^\dagger$. Let $\be_{\mathcal E(j, k)} \in \mathbb R^{|\mathcal E|}$ have one in the $\mathcal E(j, k)$th entry, and zero in the others. Then,
\begin{equation*}
    \bE^\top \bV_\beta \bZ^\dagger_{\cdot \mathcal E(j, k)} = \bE^\top \bV_\beta \bZ^\dagger \be_{\mathcal E(j, k)} = \bE^\top (\bF^\top)^\dagger \be_{\mathcal E(j, k)} = (\be_{\mathcal E(j, k)}^\top \bF^\dagger \bE)^\top.
\end{equation*}
This implies if $\gamma' \geq \frac{1}{np} \max_{(j, k) \in \mathcal E} \|[\bF^\dagger \bE]_{\mathcal E(j,k)\cdot}\|_{2}$ then $\gamma' \geq \frac{1}{np} \max_{(j, k) \in \mathcal E} \|\bE^\top \bV_\beta \bZ_{\cdot \mathcal E(j, k)}^\dagger\|_2 $ and
\begin{equation}
    \frac{1}{np}\operatorname{trace}(\bE \bV_{\beta} (\hat \bB - \bB)) \leq \gamma' \sum_{(j, k) \in\mathcal E} \|   \bZ_{\mathcal E(j, k)\cdot}(\hat \bB - \bB)\|_2 .\label{eqn:prob-bound-2}
\end{equation}
The above bound also works with $\gamma' \geq \sqrt{p} \|\bF^\dagger \bE\|_{\max}$ since $[\bF^\dagger \bE]_{\mathcal E(j,k)\cdot}$ has $p$ indices.

Combining the left-hand side of \cref{eqn:inequality_rewritten}, the bounds \cref{eqn:inequality_simplified}, \cref{eqn:hanson-wright-bound} and \cref{eqn:prob-bound-2}, we have that 
    \begin{equation*}        
    \begin{split}
        \frac{1}{2np} \|\bU - \hat \bU\|_F^2 + \gamma' \sum_{(i, j) \in \mathcal E} \| \bZ_{\mathcal E(i, j) \cdot} \hat \bB \|_2 & \leq C\sigma_p^2 \brk3{\frac{1}{n} + {\frac{\log(np)}{n p}}}  + \gamma' \sum_{(i, j) \in \mathcal E} \|\bZ_{\mathcal E(i, j) \cdot} (\hat \bB - \bB)\|_2 \\
        & \quad + \gamma' \sum_{(i, j) \in \mathcal E} \|\bZ_{\mathcal E(i, j) \cdot} \bB\|_2,
        \end{split}
    \end{equation*}
    with probability at least $1 - \frac{1}{np}$. Subtract $ \gamma' \sum_{(i, j) \in \mathcal E} \| \bZ_{\mathcal E(i, j) \cdot} \hat \bB \|_2$ from the left-hand side above and apply the triangle inequality to obtain
    \begin{displaymath}
        \frac{1}{2np} \|\bU - \hat \bU\|_F^2 \leq C\sigma_p^2\brk3{\frac{1}{n} + {\frac{\log(np)}{n p}}} + 2\gamma' \sum_{(i, j) \in \mathcal E} \|\bZ_{\mathcal E(i, j) \cdot} \bB \|_2. 
    \end{displaymath}
    Then substitute the definitions of $\gamma', \bB$ and $\bZ$,
    \begin{displaymath}
        \frac{1}{2np} \|\bU - \hat \bU\|_F^2 \leq C\sigma_p^2 \brk3{\frac{1}{n} + {\frac{\log(np)}{n p}}} + \frac{2 \gamma}{np} \sum_{(i, j) \in \mathcal E} \sqrt{\Phi_{ij}} \|\bU_{i\cdot} - \bU_{j\cdot} \|_2. 
    \end{displaymath}
\end{proof}

For completeness, we quickly prove \cref{eqn:fully-connected-bound}, the bound of \cref{thm:data-dependent-standard} for complete unweighted graphs. \cite{tanStatisticalPropertiesConvex2015} showed that $\bF^\dagger = \frac{1}{n}\bF^\top$, implying 
$\max_{(j, k) \in \mathcal E} \|\bF^\dagger_{\mathcal E(j, k)\cdot}\|_2 = \frac{\sqrt 2}{n}$. Applying \cref{corollary:independent-F-dagger-E-max} we have
\begin{displaymath}
    P\brk3{\|\bF^\dagger \bE\|_{\max} \geq \frac{4\sigma_p}{n}\sqrt{\log(np)}} \leq \frac{2}{np}.
\end{displaymath}
Then, \cref{eqn:fully-connected-bound} follows from 
\begin{displaymath}
    \sum_{(i, j) \in \mathcal E} \sqrt{\Phi_{ij}} \|\bU_{i\cdot} - \bU_{j\cdot} \|_2 \leq n^2 \diam(\mathcal C).
\end{displaymath}

\subsection{Proof of Lemma \ref{lemma:crude-bound}} 

We first show a corollary that will be useful for several proofs.
\begin{corollary}\label{corollary:independent-F-dagger-E-max}
    Let $\bE\in \mathbb R^{n \times p}$ where rows $\bE_{i\cdot}$ are independent, mean zero, and have sub-Gaussian norm $\sigma_p$. Let $\bF$ be a possibly random, but almost surely connected, oriented edge-incidence matrix as defined in \cref{def:affinity-graph}. Furthermore, let $\bE$ and $\bF$ be independent. Then
    \begin{displaymath}
         P\brk3{\|\bF^\dagger \bE\|_{\max} \geq \sqrt{8 \sigma^2_p\log(np) \max_{(j,k) \in \mathcal E} \Phi_{jk} \|\bF^\dagger_{\mathcal E(j, k)\cdot}\|_2^2} } \leq \frac{2}{np}.
    \end{displaymath}
\end{corollary}

\begin{proof}
    Each entry of $\bE_{\cdot l}$ is independent and has sub-Gaussian norm bounded by $\sigma_p$, and this remains the case for $\bE \mid \bF$. By conditioning on $\bF$, we may apply Proposition 2.7.1 of \cite{vershyninHighDimensionalProbabilityIntroduction2026} and the form described in \cref{lemma:F-dagger-form}, so $\bF^\dagger \bE_{\mathcal E(j, k), \ell} \mid \bF$ has sub-Gaussian norm bounded by $\sigma_p \sqrt{\Phi_{jk}} \|\bF^\dagger_{\mathcal E(j, k)\cdot}\|_2$. In addition, $\bE$ is mean zero, so $\mathbb E[\bF^\dagger \bE \mid \bF] = 0$. Hence, via a well-known corollary of Proposition 2.7.6 of \cite{vershyninHighDimensionalProbabilityIntroduction2026}
    \begin{displaymath}
        P\brk3{\|\bF^\dagger \bE\|_{\max} \geq \sqrt{2\sigma_p^2 \max_{(j, k)\in \mathcal E} \Phi_{jk} \|\bF^\dagger_{\mathcal E(j, k)\cdot}\|_2^2 (\log(|\mathcal E| p) + t)} \mid \bF} \leq \exp(-t). 
    \end{displaymath}
    Choose $t = \log(|\mathcal E| p)$ implying
    \begin{displaymath}
        P\brk3{\|\bF^\dagger \bE\|_{\max} \geq \sqrt{4 \sigma^2_p\log(|\mathcal E|p) \max_{(j,k) \in \mathcal E} \Phi_{jk}\|\bF^\dagger_{\mathcal E(j, k)\cdot}\|_2^2} \mid \bF} \leq \frac{1}{|\mathcal E|p}.
    \end{displaymath}

    Next, via the tower property,
    \begin{equation}
    \begin{split}
        & P\brk3{\|\bF^\dagger \bE\|_{\max} \geq \sqrt{4 \sigma^2_p\log(|\mathcal E|p)  \max_{(j,k) \in \mathcal E} \Phi_{jk} \|\bF^\dagger_{\mathcal E(j, k)\cdot}\|_2^2} } \\
        & \quad = \mathbb E\brk[s]3{P\brk3{\|\bF^\dagger \bE\|_{\max} \geq \sqrt{4 \sigma^2_p\log(|\mathcal E|p) \max_{(j,k) \in \mathcal E} \Phi_{jk} \|\bF^\dagger_{\mathcal E(j, k)\cdot}\|_2^2} \mid \bF}} \leq \mathbb E\brk[s]3{\frac{1}{|\mathcal E|p}}. \label{eqn:apply-tower-property}
    \end{split}
    \end{equation}
    Under connectedness almost surely, $ n - 1 \leq |\mathcal E| \leq n^2 / 2$. Hence, $2 \log(np) \geq \log(|\mathcal E| p)$ and $\frac{1}{|\mathcal E| p} \leq \frac{2}{np}$ almost surely. Apply this to \cref{eqn:apply-tower-property} to complete the proof.
\end{proof}

Next, we bound the entries of $\bF^\dagger$.

\begin{lemma}\label{lemma:F-dagger-max-row}
    Let $\bF$ be the oriented edge-incidence matrix for any unweighted connected graph $G = (V, \mathcal E)$. Then
    \begin{displaymath}
        \|\bF^\dagger\|_{\max} \leq 1.
    \end{displaymath}
\end{lemma}

\begin{proof}
    Via \cref{lemma:F-dagger-form} we have that $\bF^\dagger_{\mathcal E(j, k)\cdot} = \bL^\dagger_{\cdot j} - \bL^\dagger_{\cdot k} = \bL^\dagger(\be_j - \be_k)$. Hence, $\bL \bF^\dagger_{\mathcal E(j, k)\cdot} = \be_j - \be_k$. This gives for $i\notin \{j, k\}$,
    \begin{equation}
        0 = [\bL \bF^\dagger_{\mathcal E(j, k)\cdot}]_{i} = D_{ii} \bF^\dagger_{\mathcal E(j, k)i} - \sum_{\ell\colon \Phi_{\ell i} \neq 0} \bF^\dagger_{\mathcal E(j, k)\ell} \implies \bF^\dagger_{\mathcal E(j, k)i} = \frac{1}{D_{ii}}  \sum_{\ell\colon \Phi_{\ell i} \neq 0} \bF^\dagger_{\mathcal E(j, k)\ell}. \label{eqn:laplacian-on-F-dagger}
    \end{equation}
    Now assume the maximum of $\bF^\dagger_{\mathcal E(j, k)\cdot}$ occurs at index $i \notin \{j, k\}$. Then, from \cref{eqn:laplacian-on-F-dagger}
    \begin{displaymath}
        \bF^\dagger_{\mathcal E(j, k)i} = \frac{1}{D_{ii}}  \sum_{\ell\colon \Phi_{\ell i} \neq 0} \bF^\dagger_{\mathcal E(j, k)\ell} \leq \frac{1}{D_{ii}}  \sum_{\ell\colon \Phi_{\ell i} \neq 0} \bF^\dagger_{\mathcal E(j, k)i} = \bF^\dagger_{\mathcal E(j, k)i}.
    \end{displaymath}
    The assumption implies $\bF^\dagger_{\mathcal E(j, k) \ell} = \bF^\dagger_{\mathcal E(j, k)i}$ if $\ell$ is a neighbor of vertex $i$. Because $G$ is connected, this would imply $\bF^\dagger_{\mathcal E(j, k) \cdot}$ is a vector of constants. However, \cref{lemma:F-dagger-form} implies $\bF^\dagger_{\mathcal E(j, k) \cdot}$ sums to zero, which would then require all entries of $\bF^\dagger$ to be zero, a contradiction. Hence, the maximum entry of $\bF^\dagger_{\mathcal E(j, k) \cdot}$ must be at index $j$ or $k$. A similar argument shows the minimum must also occur at $j$ or $k$.
    
    The effective resistance (see \cite{ellensEffectiveGraphResistance2011} and \cite{vanmieghemPseudoinverseLaplacianBest2017} for details) between nodes $j$ and $k$ is defined as 
    \begin{displaymath}
        0 \leq R_{jk} = (\be_{j} - \be_k)^\top \bL^\dagger (\be_j - \be_k) = \bF^\dagger_{\mathcal E(j, k)j} - \bF^\dagger_{\mathcal E(j, k)k}.
    \end{displaymath}
    Theorem 2.4 of \cite{ellensEffectiveGraphResistance2011} has that $R_{jk}$ is less than or equal to the size of the path between the two nodes. However, vertices $j$ and $k$ have an edge between them, so $R_{jk} = \bF^\dagger_{\mathcal E(j, k)j} - \bF^\dagger_{\mathcal E(j, k)k} \leq 1$. These are the maximum and minimum elements of $\bF^\dagger_{\mathcal E(j, k)\cdot}$ so $|\bF^\dagger_{\mathcal E(j, k)i} - \bF^\dagger_{\mathcal E(j, k)\ell}| \leq 1$ for all pairs of nodes. Now assume $\bF^\dagger_{\mathcal E(j, k)i} > 1$ for some index $i$. Since, $\bF^\dagger_{\mathcal E(j, k)\cdot}$ sums to 0, then $\bF^\dagger_{\mathcal E(j, k)i} = -\sum_{\ell \neq i} \bF^\dagger_{\mathcal E(j, k)i}$. However, this implies negative numbers in $\bF^\dagger_{\mathcal E(j, k)\cdot}$ and some $\ell$ where $|\bF^\dagger_{\mathcal E(j, k)i} - \bF^\dagger_{\mathcal E(j, k)\ell}| > 1$. This contradiction can be repeated in the other direction, showing $|\bF^\dagger_{\mathcal E(j, k)i}| \leq 1$ for all $i$ and all $j, k$ sharing an edge.
\end{proof}

Finally, the proof of \cref{lemma:crude-bound} follows immediately from \cref{corollary:independent-F-dagger-E-max}, \cref{lemma:F-dagger-max-row} and the definition of $\mathcal I$.

\subsection{Proof of Lemma \ref{lemma:F-dagger-form}}

\begin{proof}
    By construction, $\bF^\dagger = \bF^\top (\bF \bF^\top)^\dagger = \bF^\top \bL^\dagger$. Recall the dimension of $\bF^\dagger$ is $|\mathcal E| \times n$, so each entry corresponds to an edge and a node. Let $\mathcal E(j, k)$ be the index of the edge between nodes $j$ and $k$ such that $\bF_{j, \mathcal E(j, k)} = -\bF_{k, \mathcal E(j, k)} = \sqrt{\Phi_{jk}}$. Each column has only two nonzero elements, hence
    \begin{displaymath}
        \bF^\dagger_{\mathcal E(j, k), i} = \operatorname{row}^\top_{\mathcal E(j, k)}(\bF^\top)\operatorname{col}_{i}(\bL^\dagger) = \sum_{\ell=1}^n \bF_{\ell, \mathcal E(j, k) } \bL^\dagger_{i\ell} = \sqrt{\Phi_{jk}}(L_{ij}^\dagger - L_{ik}^\dagger).
    \end{displaymath}
    Equation (B2) of \cite{vanmieghemPseudoinverseLaplacianBest2017} describes elements of the Laplacian pseudoinverse in terms of commute times,
    \begin{displaymath}
        L_{ij}^\dagger = \frac{1}{2\vol(G)}\brk3{- C_{ij} + \frac{1}{n}\sum_{\ell=1}^n C_{i\ell} + C_{j\ell}} - R(G),
    \end{displaymath}
    where $R(G)$ is a constant depending on $G$. Standard algebra completes the proof statement.
\end{proof}

\subsection{Proof of Lemma \ref{lemma:hitting-time-bound}} \label{appendix:lemma:hitting-time-bound}

\begin{proof}
    Via \cref{lemma:F-dagger-form} each entry of $\bF^\dagger$ can be written as
    \begin{displaymath}
        \bF^\dagger_{\mathcal E(j, k), i} = \frac{\sqrt{\Phi_{jk}}}{2\vol(G)} \brk{C_{ik} - C_{ij} + \bar \bC_{\cdot j} - \bar \bC_{\cdot k}}.
    \end{displaymath}
    Let $\eta_{\ell m} = H_{\ell m} - \frac{\vol(G)}{D_{mm}}$ for $\ell \neq m$ and $|\eta_{\max}| = \max_{\ell \neq m} |\eta_{\ell m}|$. Recall $H_{ii} = 0$ for all $i$. We split the calculation up into two parts. First, the mean terms
    \begin{equation}
        \bar \bC_{\cdot j} = \bar \bH_{\cdot j} + \bar \bH_{j \cdot} = \frac{1}{n} \sum_{\ell \neq j} \frac{\vol(G)}{D_{jj}} + \frac{\vol(G)}{D_{\ell\ell}} + \eta_{\ell j} + \eta_{j \ell}, \label{eqn:commute-mean-diff}
    \end{equation}
    which combine to give
    \begin{displaymath}
        \bar \bC_{\cdot j} - \bar \bC_{\cdot k} = \frac{n-2}{n} \brk3{ \frac{\vol(G)}{D_{jj}} - \frac{\vol(G)}{D_{kk}} } + \frac{1}{n}\sum_{\ell \notin \{j, k\}} \eta_{\ell j} + \eta_{j\ell} - \eta_{\ell k} - \eta_{k \ell}.
    \end{displaymath}
    Next, we have
    \begin{align}
        H_{ik} + H_{ki} - H_{ij} - H_{ji} & = \frac{\vol(G)}{D_{kk}} - \frac{\vol(G)}{D_{jj}} + \eta_{ik} + \eta_{ki} - \eta_{ij} - \eta_{ji}, \quad (i \notin \{j, k\}); \label{eqn:hitting-time-diff} \\
        H_{ik} + H_{ki} - H_{ij} - H_{ji} & = \frac{\vol(G)}{D_{kk}} + \frac{\vol(G)}{D_{jj}} + \eta_{ik} + \eta_{ki}, \quad (i = j) \label{eqn:hitting-time-diff-i-eqaul-j}.
    \end{align}
    Hence, for $i \notin \{j, k\}$, by summing \cref{eqn:commute-mean-diff} and \cref{eqn:hitting-time-diff}, then applying the triangle inequality to all $\eta$ terms
    \begin{align*}
        \bF^\dagger_{\mathcal E(j, k), i} & = \frac{\sqrt{\Phi_{jk}}}{2\vol(G)}\brk3{\frac{2\vol(G)}{n} \brk3{ \frac{1}{D_{kk}} - \frac{1}{D_{jj}} } + \frac{1}{n}\sum_{\ell \notin \{i, j, k\}} \eta_{\ell j} + \eta_{j\ell} - \eta_{\ell k} - \eta_{k \ell}}, \\
        |\bF^\dagger_{\mathcal E(j, k), i}| & \leq \sqrt{\Phi_{jk}}\brk3{\frac{1}{n} \bigg|{\frac{1}{ D_{kk}} - \frac{1}{ D_{jj}}}\bigg| + 2\max_{\ell \neq m} \bigg|\frac{H_{\ell m}}{\vol(G)} - \frac{1}{D_{mm}}\bigg|}.
    \end{align*}
    Similarly for $i = j$ and summing \cref{eqn:commute-mean-diff} and \cref{eqn:hitting-time-diff-i-eqaul-j},
    \begin{align*}
        \bF^\dagger_{\mathcal E(j, k), j} & = \frac{\sqrt{\Phi_{jk}}}{2\vol(G)}\brk3{\frac{2\vol(G)}{n}\brk3{\frac{n-1}{D_{jj}} + \frac{1}{D_{kk}}} +  \eta_{jk} + \eta_{kj} + \frac{1}{n}\sum_{\ell \notin \{j, k\}} \eta_{\ell j} + \eta_{j\ell} - \eta_{\ell k} - \eta_{k \ell}}  \\
        |\bF^\dagger_{\mathcal E(j, k), j}| & \leq \sqrt{\Phi_{jk}} \brk3{\frac{1}{D_{jj}} + \frac{1}{nD_{kk}} + 2 \max_{\ell \neq m} \bigg|\frac{H_{\ell m}}{\vol(G)} - \frac{1}{D_{mm}}\bigg| }.
    \end{align*}
\end{proof}

\subsection{Proof of Corollary \ref{corollary:oracle-graph}}

\begin{proof}
    From \cref{corollary:independent-F-dagger-E-max} we can get a high probability bound with the maximum value of $\|\bL_{\cdot j}^\dagger - \bL_{\cdot k}^\dagger\|_2$ for the constructed $\bPhi$. Typically this would be non-trivial, but via \cref{lemma:F-dagger-form} we only need to calculate the hitting times for all pairs of nodes. Examining \cref{fig:oracle-graph}, via symmetry, there is a fixed number of unique hitting times for all graphs constructed with $n \geq 2k$ and $k \geq 2$. Precisely, there are nine unique hitting times which can be calculated by inverting a $9 \times 9$ matrix constructed using basic techniques. We completed these calculations using \textit{Mathematica} and verified them numerically.

    Matching the node labels in \cref{fig:oracle-graph}, there are only four possible values for $\Phi_{jk}\| \bL_{\cdot j}^\dagger - \bL_{\cdot k}^\dagger\|_2^2$ corresponding to each type of edge in the graph: $(A \leftrightarrow F)$, $(A \leftrightarrow D)$, $(D \leftrightarrow E)$, and $(C \leftrightarrow D)$. In closed form,
    \begin{align*}
        \|\bL^\dagger_{\cdot A} - \bL^\dagger_{\cdot E}\|_2^2 & = \frac{16k(n + 2) + n (n + 4)^2 + 8 k ^2 (n^2 +6n + 12)}{k^2n^2(n+4)^2}, \\
        \|\bL^\dagger_{\cdot D} - \bL^\dagger_{\cdot C}\|_2^2 & = \frac{-32k + 32k^2 + n(n + 4)^2}{4k^2(n+4)^2}, \\
        \|\bL^\dagger_{\cdot D} - \bL^\dagger_{\cdot E}\|_2^2 & = \frac{8(n^2 + 4n + 8)}{n^2(n + 4)^2}, \quad k \geq 2, \\
        \|\bL^\dagger_{\cdot A} - \bL^\dagger_{\cdot F}\|_2^2 & = \frac{8}{n^2}, \quad k \leq n/2 - 2,
    \end{align*}
    where the maximum occurs for edges that are bridges between the two complete subgraphs. Apply \cref{corollary:independent-F-dagger-E-max} for
    \begin{displaymath}
        P\brk3{\|\bF^\dagger \bE\|_{\max} \geq \frac{\sqrt{2\sigma_p^2 \log(n p) (n(n+4)^2+32k^2-32k)}}{k(n+4)}} \leq \frac{2}{n p}.
    \end{displaymath}
    Finally, the following identity for $n \geq 1, 1 \leq k \leq n/2$ simplifies the probability,
    \begin{displaymath}
        \sqrt{n} + \frac{\sqrt{32}k}{n} \geq \frac{\sqrt{n(n+4)^2 + 32k^2 - 32k}}{n+4},
    \end{displaymath}
    and, as constructed, $\sum_{(i, j) \in \mathcal E} \sqrt{\Phi_{ij}} \|\bU_{i\cdot} - \bU_{j\cdot}\|_2 = k \|\bc_1 - \bc_2\|_2$.
\end{proof}

\subsection{Proof of Theorem \ref{thm:edge-crossings-bound}}\label{appendix:thm:edge-crossings-bound}

\begin{proof}
    Let $\mathcal I = \brk[c]{(j, k) \in \mathcal E \colon \bU_{j\cdot} \neq \bU_{k\cdot}}$ be the set of between-cluster edges. Each edge between-clusters is independent and occurs with probability $p_b$, so $|\mathcal I| \sim Binom(n^2 / 2, p_b)$. Using a well-known concentration inequality for binomial random variables (e.g. Proposition 27 of \cite{luxburgHittingCommuteTimes2014}),
    \begin{displaymath}
        P\brk3{|\mathcal I| \geq (1 + \delta) \frac{n^2 p_b}{2}} \leq \exp\brk3{-\frac{\delta^2 n^2 p_b}{6}}.
    \end{displaymath}
    Hence, choose $\delta = 1/2$ and bound
    \begin{equation}
        P\brk3{\sum_{(j, k) \in \mathcal E} \sqrt{\Phi_{jk}} \|\bU_{i\cdot} - \bU_{j\cdot}\|_2 \geq \frac{3 n^2 p_b}{4} \|\bc_1 - \bc_2\|_2} \leq \exp\brk3{-\frac{n^2 p_b}{24}}. \label{eqn:cut-term-concentration}
    \end{equation}

    Next, we bound $\|\bF^\dagger_{\mathcal E(j, k)\cdot}\|_2$ with high probability. Let $|\eta_{\max}| = |\frac{H_{\ell m}}{\vol(G)} - \frac{1}{D_{mm}}|$ and $d_{\min} = \min_{i} D_{ii}$. Via \cref{lemma:hitting-time-bound} we have
    \begin{align*}
        \max_{(j, k) \in \mathcal E} \|\bF^\dagger_{\mathcal E(j, k)\cdot}\|_2 & = \max_{(j, k) \in \mathcal E} \sqrt{(\bF^\dagger_{\mathcal E(j, k), j})^2 + (\bF^\dagger_{\mathcal E(j, k), k})^2 + \sum_{i\notin \{j, k\}} (\bF^\dagger_{\mathcal E(j, k), i})^2}, \\
        & \leq \sqrt{2\brk3{\frac{2}{d_{\min}} + 2 |\eta_{\max}|}^2 + n \brk3{\frac{2}{nd_{\min}} + 2 |\eta_{\max}|}^2}, \quad (D_{ii}^{-1} \leq d_{\min}^{-1}), \\
        & \leq \sqrt{2}\brk3{\frac{2}{d_{\min}} + 2 |\eta_{\max}|} + \sqrt{n} \brk3{\frac{2}{nd_{\min}} + 2 |\eta_{\max}|}, \quad (\sqrt{a^2 + b^2} \leq |a| + |b|), \\
        & = \brk3{2\sqrt{2} + \frac{2}{\sqrt{n}}}\frac{1}{d_{\min}} + \brk{2\sqrt{n} + 2}|\eta_{\max}|, \\
        & \leq 4\brk3{\frac{1}{d_{\min}} + \sqrt{n}|\eta_{\max}|}, \quad (n \geq 2).
    \end{align*}
    Apply \cref{prop:luxberg} for
    \begin{equation}
        \max_{(j, k) \in \mathcal E} \|\bF^\dagger_{\mathcal E(j, k)\cdot}\|_2 \leq 4\brk3{\frac{1}{d_{\min}} + \frac{2\sqrt{n}}{d_{\min}^2} \brk3{\frac{1}{1-\lambda_2} + 1}}, \label{eqn:F-dagger-2-norm-max}
    \end{equation}
    where $\lambda_2$ is the second largest eigenvalue of the transition matrix $\bD^{-1} \bPhi$.

    Now, we show a concentration inequality for $d_{\min}$. Each $D_{ii} = X_i + Y_i$, where $X_i \sim Binom(n / 2 - 1, p_w)$ and $Y_i \sim Binom(n / 2, p_b)$. Let $\delta \in (0, 1)$ and do several union bounds for
    \begin{align*}
        P\brk3{d_{\min} \leq \frac{(1 - \delta) (n(p_w + p_b) -  p_w)}{2}} & \leq P\brk3{\bigcup_{i=1}^n D_{ii} \leq \frac{(1 - \delta)( n(p_w + p_b) - p_w)}{2}}, \\
        & \leq \sum_{i=1}^n P\brk3{D_{ii} \leq \frac{(1 - \delta)( n(p_w + p_b) -  p_w)}{2}}, \\
        & \leq \sum_{i=1}^n P\brk3{X_i \leq  \frac{(1 - \delta )(n-1) p_w}{2}} \\
        & \qquad + P\brk3{Y_i \leq  \frac{(1 - \delta) n p_b}{2}}. 
    \end{align*}
    Then apply Proposition 27 of \cite{luxburgHittingCommuteTimes2014} for
    \begin{displaymath}
        P\brk3{d_{\min} \leq \frac{(1 - \delta) (n(p_w + p_b) -  p_w)}{2}} \leq n \exp\brk3{-\frac{\delta^2 (n-1) p_w}{6}} + n\exp\brk3{-\frac{\delta^2 n p_b}{6}}.
    \end{displaymath}
    For $n \geq 2$, $n(p_w + p_b) - p_w > n(p_w + p_b)/2$, which combined with the above and $\delta = \frac{1}{2}$ gives
    \begin{equation}
        P\brk3{d_{\min} \leq \frac{n(p_w + p_b)}{8}} \leq n \exp\brk3{-\frac{(n-1) p_w}{24}} + n\exp\brk3{-\frac{n p_b}{24}}. \label{eqn:clean-dmin-concentration}
    \end{equation}

    Let $0 < \epsilon < 1$ and $\mathbb E[D_{ii}] = \frac{n(p_w + p_b) - p_w}{2}$ be the expected degree of all nodes in $G$. Similar to the steps in the proof for Corollary 14 of \cite{luxburgHittingCommuteTimes2014}, we calculate the following key quantities to apply Theorem 2 of \cite{chungSpectraGeneralRandom2011}:
    \begin{align*}
        \mathbb E[\bPhi] & = \begin{pmatrix}
        p_w & p_b \\ p_b & p_w
    \end{pmatrix} \otimes (\mathbf 1_{n/2} \mathbf 1_{n/2}^\top) - p_w\bI_{n}, \\
        \bA & = \mathbb E[\bD]^{-1/2} \mathbb E[\bPhi] \mathbb E[\bD]^{-1/2} = \frac{1}{\mathbb E[D_{ii}]} \mathbb E[\bPhi],
    \end{align*}
    where $\otimes$ is the Kronecker product. The unique eigenvalues of $\mathbb E[\bPhi]$ and $\bA$ are 
    \begin{align*}
        \lambda_1(\mathbb E[\bPhi]) & = \frac{n(p_w + p_b)}{2} - p_w = \mathbb E[D_{ii}] - \frac{p_w}{2}, \quad & \lambda_1(\bA) = 1 - \frac{p_w}{n(p_w + p_b) - p_w}, \\
        \lambda_2(\mathbb E[\bPhi]) & = \frac{n(p_w - p_b)}{2} - p_w = -\mathbb E[D_{ii}] + np_w - \frac{3p_w}{2}, \quad & \lambda_2(\bA) = -1 + \frac{2np_w - 3p_w}{n(p_w + p_b) - p_w}, \\
        \lambda_3(\mathbb E[\bPhi]) & = - p_w, \quad & \lambda_3(\bA) = -\frac{2p_w}{n(p_w + p_b) - p_w}.
    \end{align*}
    Clearly $\lambda_1(\bA) > \lambda_2(\bA)$ and since $n p_w \to \infty$ and $p_w = \omega(p_b)$, $\lambda_2(\bA) \to 1$ and $\lambda_3(\bA) \to 0$. As a result, because $n p_w = \omega\brk{\log(n)}$, we may apply Theorem 2 of \cite{chungSpectraGeneralRandom2011} to bound the spectral gap, so for all $\epsilon > 0$, there exists $N(\epsilon)$ such that 
    \begin{displaymath}
        |(1 - \lambda_2) - (1 - \lambda_2(\bA))| = |(1 - \lambda_2) - \frac{2np_b + p_w}{n(p_w + p_b) - p_w}| \leq 2 \sqrt{\frac{3\log(4n / \epsilon)}{\mathbb E[D_{ii}]}},
    \end{displaymath}
    with probability at least $1-\epsilon$ for all $n > N(\epsilon)$. Under the growth rates for $p_b$ and $p_w$,
    \begin{displaymath}
        \frac{p_b}{p_w + p_b} = \omega\brk3{\sqrt{\frac{\log(n)}{n(p_w + p_b)}}},
    \end{displaymath} 
    so we may let $n$ also be large enough that
    \begin{equation}
        2 \sqrt{\frac{6\log(4n / \epsilon)}{n(p_w + p_b)}} < \frac{2np_b + p_w}{2n(p_w + p_b) - 2p_w} \implies \frac{1}{1-\lambda_2} \leq \frac{n(p_w + p_b) - p_w}{np_b + p_w / 2} \leq 1 + \frac{p_w}{p_b}. \label{eqn:spectral-gap-bound}
    \end{equation}

    Finally, we combine the intermediate results. A union bound with \cref{corollary:independent-F-dagger-E-max} and \cref{eqn:cut-term-concentration} gives
    \begin{equation}
        \frac{\|\bF^\dagger \bE\|_{\max}}{n} \sum_{(i, j) \in \mathcal E} \sqrt{\Phi_{ij}} \|\bU_{i\cdot} - \bU_{j\cdot}\|_2 \leq \frac{3}{4}\sigma_p \max_{(j, k) \in \mathcal E} \|\bF^\dagger_{\mathcal E(j, k)\cdot}\|_2 \sqrt{8\log(np)}n p_b\|\bc_1 - \bc_2\|_2, \label{eqn:cut-and-error-bound}
    \end{equation}
    with probability at least $ 1 - \frac{2}{np} - \exp\brk1{-\frac{n^2 p_b}{24}}$. Continuing, apply \cref{eqn:clean-dmin-concentration} and \cref{eqn:spectral-gap-bound} to \cref{eqn:F-dagger-2-norm-max} with a union bound giving
    \begin{equation}
        \max_{(j, k) \in \mathcal E} \|\bF^\dagger_{\mathcal E(j, k)\cdot}\|_2 \leq 4\brk3{\frac{8}{n(p_w + p_b)} + \frac{64\sqrt{n}}{n^2(p_w + p_b)^2} \brk3{2 + \frac{p_w}{p_b}}}, \label{eqn:max-2-norm-concentration} \\
    \end{equation}
    with probability at least $1 - \epsilon - n \exp\brk1{-\frac{(n-1) p_w}{24}} + n\exp\brk1{-\frac{n p_b}{24}}$ for $n > N(\epsilon)$.

    Combining \cref{eqn:cut-and-error-bound} and \cref{eqn:max-2-norm-concentration} and simplifying leads to
    \begin{align*}
        \frac{\|\bF^\dagger \bE\|_{\max}}{n} \sum_{(i, j) \in \mathcal E} \sqrt{\Phi_{ij}} \|\bU_{i\cdot} - \bU_{j\cdot}\|_2 \leq C \sigma_p \sqrt{\log(np)} \|\bc_1 - \bc_2\|_2 \brk3{\frac{p_b}{p_w + p_b} + \frac{1}{\sqrt{n}(p_w + p_b)}},
    \end{align*}
    for some universal constant $C > 0$. This statement occurs with probability arbitrarily close to 1 as $n\to \infty$ since $n p_b \to \infty$, $np_w \to \infty$, and $\epsilon$ was arbitrary.
\end{proof}

\bibliographystyle{siamplain}
\bibliography{references}

\end{document}